\documentclass{article} 
\usepackage{times}
\usepackage[top=1in, bottom=1in, left=0.7in, right=0.7in]{geometry}
\usepackage{graphicx} 

\usepackage[T1]{fontenc}
\usepackage[utf8]{inputenc}

\usepackage{algorithm}
\usepackage{algorithmic}

\usepackage{subfigure}
\usepackage{wrapfig}
\usepackage{natbib}
\usepackage{caption}
\usepackage{epsf}
\usepackage{amsmath,mathtools}
\usepackage{amssymb}
\usepackage{amsfonts}
\usepackage{amsthm}
    
\newlength{\minipagewidth}
\newlength{\minipagewidthx}
\newcommand{\bookboxx}[1]{\small
\par\medskip\noindent
\framebox[\columnwidth]{
\begin{minipage}{0.95\columnwidth} {#1} \end{minipage} } \par\medskip }

\newtheorem{asm}{Assumption}

\title{Sequential Transfer in Multi-armed Bandit\\ with Finite Set of Models}

\author{Mohammad Gheshlaghi Azar \\  School of Computer Science\\ Carnegie Mellon University\\ \normalsize{\texttt{mazar@cs.cmu.edu} }\and Alessandro Lazaric \\ INRIA Lille - Nord Europe\\  Villeneuve d'Ascq, France \\\normalsize {\texttt{alessandro.lazaric@inria.fr}} \and Emma Brunskill  \\  School of Computer Science\\ Carnegie Mellon University\\ \normalsize{\texttt{ebrun@cs.cmu.edu} }}

\newcommand{\A}{\mathcal A}

\newcommand{\E}{\mathbb E}
\newcommand{\calE}{\mathcal E}

\newcommand{\Prob}{\mathbb P}

\newcommand{\hmu}{\hat\mu}

\newcommand{\R}{\mathcal{R}}

\newcommand{\eps}{\varepsilon}

\newcommand{\htheta}{\hat{\theta}}
\newcommand{\hi}{\hat{i}}
\newcommand{\btheta}{\bar{\theta}}
\newcommand{\tTheta}{\widetilde{\Theta}}

\newtheorem{lemma}{Lemma}
\newtheorem{corollary}{Corollary}

\newtheorem{definition}{Definition}
\newtheorem{theorem}{Theorem}

\begin{document}
\maketitle

\vspace{-0.1in}
\begin{abstract}
Learning from prior tasks and transferring that experience to improve future performance is critical for building lifelong learning agents. Although results in supervised and reinforcement learning show that transfer may significantly improve the learning performance, most of the literature on transfer is focused on batch learning tasks. In this paper we study the problem of \textit{sequential transfer in online learning}, notably in the multi-armed bandit framework, where the objective is to minimize the cumulative regret over a sequence of tasks by incrementally transferring knowledge from prior tasks. 
We introduce a novel bandit algorithm based on a method-of-moments approach for the estimation of the possible tasks and derive regret bounds for it. 
\end{abstract}


\vspace{-0.2in}
\section{Introduction}\label{s:introduction}
\vspace{-0.1in}

Learning from prior tasks and transferring that experience to improve future performance is a key aspect of intelligence, and is critical for building lifelong learning agents. 
Recently, multi-task and transfer learning received much attention in the supervised and reinforcement learning (RL) setting with both empirical and theoretical encouraging results \citep[see recent surveys by][]{pan2010a-survey,lazaric2011transfer2}. Most of these works focused on scenarios where the tasks are batch learning problems, in which a training set is directly provided to the learner. 
On the other hand, the online learning setting \citep{cesa2006prediction}, where the learner is presented with samples in a sequential fashion, has been rarely considered (see~\citet{Mann12Directed} for an example in RL and Sec.~\ref{s:related} in the supplementary material for a discussion on related settings). 

The multi--arm bandit (MAB)~\citep{robbins1952some} is a simple yet powerful framework formalizing the online learning with partial feedback problem, which encompasses a large number of applications, such as clinical trials, online advertisements and adaptive routing. 
In this paper we take a step towards understanding and providing formal bounds on transfer in stochastic MABs. We focus on a \textit{sequential transfer} scenario where an (online) learner is acting in a series of  tasks drawn from a stationary distribution over a finite set of MABs.
Prior to learning, the model parameters of each  bandit problem  are not known to the learner, nor does it know the distribution probability over the bandit problems. Also, we assume that the learner is not  provided with the identity of the task. This setting is sufficient to model a number of interesting problems,  including: a tutoring system working to help a sequence 
of students to learn, by finding the right type of education program for them, where each student may be a remedial, normal or honors  student but which is unknown; an online advertisement site that wishes to run a sequence of ads with maximum expected click  for a sequence of webpages based on the type of the users of each webpage, which is unknown to the system. 
%
%
 %
To act  efficiently in this setting,  it is crucial to define a mechanism for  transferring knowledge across tasks.   In fact
the learner may encounter the same bandit problem over and over throughout the learning, and
an efficient algorithm should be able to reuse (transfer) the knowledge obtained in previous tasks, when it is presented with the same problem again. This can be achieved by modeling the  reward distribution of the whole process as a latent variable model (LVM), where the observed variables are the rewards of pulling the arms and the  latent variable is the  identity of the bandit.  If we can accurately estimate this LVM, we show that an extension of the \textit{UCB} algorithm~\citep{auer2002finite-time} is able to exploit this prior knowledge   to reduce the regret through tasks (Sec.~\ref{s:mucb}). 

In this paper we rely on a new variant of  method-of-moments~\citep{AnandkumarFHKL12,AnandkumarHK12}, the robust tensor power method (RTP)~\citep{AnandkumarRHKT12}, to estimate the LVM associated with the sequential-bandit problem. \textit{RTP} relies on  decomposing the eigenvalues/eigenvectors of certain tensors for estimating the model means~\citep{AnandkumarRHKT12}. We prove that  \textit{RTP} provides  a consistent estimate of the means of all  arms for every bandit problem as long as they are pulled at least three times per task (Sec.~\ref{ss:mom}). This guarantees that once \textit{RTP} is paired with an efficient bandit algorithm able to exploit the transferred knowledge about the models (Sec.~\ref{ss:umucb.regret}), we obtain a bandit algorithm, called \textit{tUCB}, guaranteed to perform as well as \textit{UCB} in early episodes, thus avoiding any negative transfer effect, and then to approach the performance of the ideal case when the set of bandit problems is known in advance (Sec.~\ref{ss:tucb.regret}). Finally, we report some preliminary results on synthetic data confirming the theoretical findings (Sec.~\ref{s:experiments}).
%
%
%
\vspace{-0.1in}
\section{Preliminaries}\label{s:preliminaries}
\vspace{-0.1in}
We consider a stochastic MAB problem defined by a set of arms $\A = \{1,\ldots,K\}$, $|\A|=K$, where each $i\in\A$ is characterized by a distribution $\nu_i$  and the samples observed from each arm are independent and identically distributed. We focus on the setting where there exists a set of models $\Theta = \{\theta = (\nu_1,\ldots,\nu_K)\}$, $|\Theta|=m$, which contains all the possible bandit problems. We denote the mean of an arm $i$, the best arm, and the best value of a model $\theta\in\Theta$ respectively by $\mu_i(\theta)$, $i_*(\theta)$, $\mu_*(\theta)$. We define the arm gap of an arm $i$ for a model $\theta$ as $\Delta_i(\theta) = \mu_*(\theta) - \mu_i$, while the model gap for an arm $i$ between two models $\theta$ and $\theta'$ is defined as $\Gamma_i(\theta, \theta') = |\mu_i(\theta)-\mu_i(\theta')|$. 

We also introduce some tensor notation. Let $X\in\mathbb{R}^K$ be a random realization of all the arms from a random model. All the realizations are i.i.d. conditional on a model $\btheta$ and $\E[X| \theta=\btheta] = \mu(\theta)$, where the $i$-th component of $\mu(\theta)\in\mathbb{R}^K$ is $[\mu(\theta)]_i=\mu_i(\theta)$. Given two realizations $X^1$ and $X^2$, we define the second moment matrix $M_2 = \E[X^1 \otimes X^2]$ such that $[M_2]_{i,j} = \E[X_i^1 X_j^2]$ and the third moment tensor $M_3 = \E[X^1 \otimes X^2 \otimes X^3]$. Since the realizations are conditionally independent, we have that $\E[X^1\otimes X^2| \theta=\btheta] = \E[X^1|\theta=\btheta] \otimes \E[X^2|\theta=\btheta] = \mu(\theta) \otimes \mu(\theta)$ and this allows us to rewrite the second and third moments as $M_2 = \sum_{\theta} \rho(\theta) \mu(\theta)^{\otimes 2}, M_3 = \sum_{\theta} \rho(\theta) \mu(\theta)^{\otimes 3}$~\citep{AnandkumarHK12}, where ${v}^{\otimes p} = {v} \otimes {v} \otimes \cdots {v}$ is the $p$-th tensor power.
Let  $A$ be a $3^{\text{rd}}$ order member of the tensor product of the Euclidean space $\mathbb R^K$ (as $M_3$), then we define the multilinear map as follows. For a set of three matrices $\{V_i \in  \mathbb R^ {K\times m}\}_{1\leq i\leq 3}$ , the $(i_1,i_2,i_3)$ entry in the $3$-way  array representation of $A(V_1,V_2,V_3)\in\mathbb R^{m\times m\times m}$ is
$[A(V_1,V_2,V_3)]_{i_1,i_2,i_3}:=\sum_{1\leq j_1,j_2,j_3\leq n} A_{j_1,j_2,j_3}[V_1]_{j_1,i_1}[V_2]_{j_2,i_2}[V_3]_{j_3,i_3}$.
%
We also use different norms: the Euclidean norm $\|\cdot\|$; the Frobenius norm $\|\cdot\|_F$; the matrix max-norm $\|A\|_{\max}=\max_{ij}|[A]_{ij}|$.

We consider the sequential transfer setting where at each episode $j$ the learner interacts with a task $\btheta^j$, drawn from a distribution $\rho$ over $\Theta$, for $n$ steps. The objective is to minimize the (pseudo-)regret over $J$ episodes measured as the difference between the rewards obtained by the optimal arms $i_*(\btheta^j)$ and the rewards achieved by the learner. More formally, the regret is defined as
\begin{equation}\label{eq:global.regret}
\R_J = \sideset{}{_{j=1}^J }{\sum}\R^j_n = \sideset{}{_{j=1}^J}\sum \sideset{}{_{i\neq i^*}}\sum T_{i,n}^j \Delta_i(\btheta^j),
\end{equation}
where $T_{i,n}^j$ is the number of pulls to arm $i$ after $n$ steps of episode $j$. The only information available to the learner is the number of models $m$,  number of episodes $J$ and  number of steps $n$ per task. 
%
%
\vspace{-0.1in}
\section{Mult-armed Bandit with Finite Models}\label{s:mucb}
\vspace{-0.1in}
%
%

%
%
%
\begin{wrapfigure}{r}{0.5\linewidth}
\vspace{-0.25in}
\begin{minipage}[t]{1.0\linewidth}
{\bookboxx{
\begin{algorithmic}
\REQUIRE Set of models $\Theta$, number of steps $n$
\FOR{$t = 1,\ldots,n$}
\STATE Build $\Theta_t = \{\theta:\forall i, |\mu_i(\theta) - \hmu_{i,t}| \leq \eps_{i,t}\}$
\STATE Select $\theta_t = \arg\max_{\theta\in\Theta_t} \mu_*(\theta)$
\STATE Pull arm $I_t = i_*(\theta_t)$
\STATE Observe sample $x_{I_t}$ and update
\ENDFOR
\end{algorithmic}
}}
\par\vspace{-0.15in}
\caption{The \textit{mUCB} algorithm.}
\label{f:mucb}
\par\vspace{-0.15in}
\end{minipage}
\end{wrapfigure}
Before considering the transfer problem, we show that a simple variation to \textit{UCB} allows to effectively exploit the knowledge of $\Theta$ and obtain a significant reduction in the regret. The \textit{mUCB} (model-UCB) algorithm in Fig.~\ref{f:mucb} takes as input a set of models $\Theta$ including the current (unknown) model $\btheta$. At each step $t$, the algorithm computes a subset $\Theta_t \subseteq \Theta$ containing only the models whose means $\mu_i(\theta)$ are \textit{compatible} with the current estimates $\hmu_{i,t}$ of the means $\mu_i(\btheta)$ of the current model, obtained averaging $T_{i,t}$ pulls, and their uncertainty $\eps_{i,t}$ (see Eq.~\ref{eq:eps.def} for an explicit definition of this term). Notice that it is enough that one arm does not satisfy the compatibility condition to discard a model $\theta$.
Among all the models in $\Theta_t$, \textit{mUCB} first selects the model with the largest optimal value and then it pulls its corresponding optimal arm. This choice is coherent with the \textit{optimism in the face of uncertainty} principle used in UCB-based algorithms, since \textit{mUCB} always pulls the optimal arm corresponding to the optimistic model compatible with the current estimates $\hmu_{i,t}$. We show that \textit{mUCB} incurs a regret which is never worse than \textit{UCB} and it is often significantly smaller.
%
%
%

We denote the set of arms which are optimal for at least a model in a set $\Theta'$ as $\A_*(\Theta') = \{i\in\A: \exists \theta\in\Theta': i_*(\theta)=i\}$. The set of models for which the arms in $\A'$ are optimal is $\Theta(\A') = \{\theta\in\Theta: \exists i\in\A': i_*(\theta)=i\}$. The set of optimistic models for a given model $\btheta$ is $\Theta_+ = \{\theta\in\Theta: \mu_*(\theta)\geq \mu_*(\btheta)\}$, and their corresponding optimal arms $\A_+ = \A_*(\Theta_+)$. 
The following theorem bounds the expected regret (similar bounds hold in high probability). The lemmas and proofs (using standard tools from the bandit literature) are available in Sec.~\ref{app:proofs.mucb} of the supplementary material.
\begin{theorem}\label{thm:mucb.regret}
If \textit{mUCB} is run with $\delta=1/n$, a set of $m$ models $\Theta$ such that the $\btheta\in\Theta$ and 
%
\begin{align}\label{eq:eps.def}
\eps_{i,t} = \sqrt{\log(mn^2/\delta)/(2T_{i,t-1})},
\end{align}
%
where $T_{i,t-1}$ is the number of pulls to arm $i$ at the beginning of step $t$, then its expected regret is
\begin{align}\label{eq:mucb.regret}
\E[\R_n] \leq K +\sideset{}{_{i\in \A_+}} \sum  \frac{2\Delta_i(\btheta)\log \big(mn^3\big)}{\min_{\theta\in\Theta_{+,i}}\Gamma_i(\theta,\btheta)^2}  \leq K +\sideset{}{ _{i\in \A_+}}\sum  \frac{2\log \big(m n^3\big)}{\min_{\theta\in\Theta_{+,i}}\Gamma_i(\theta,\btheta)},
\end{align}
where $\A_+=\A_*(\Theta_+)$ is the set of arms which are optimal for at least one optimistic model $\Theta_+$ and $\Theta_{+,i} = \{\theta\in\Theta_+: i_*(\theta)=i\}$ is the set of optimistic models for which $i$ is the optimal arm.
\end{theorem}
\textbf{Remark (comparison to \textit{UCB}).}
The \textit{UCB} algorithm incurs a regret
\begin{align*}
\E[\R_n(\text{UCB})] \leq O\Big(\sideset{}{_{i\in\A}}{\sum}\frac{\log n}{\Delta_i(\btheta)}\Big) \leq O\Big(K\frac{\log n}{\min_i\Delta_i(\btheta)}\Big).
\end{align*}
%
%
We see that \textit{mUCB} displays two major improvements. The regret in Eq.~\ref{eq:mucb.regret} can be written as
\begin{align*}
\E[\R_n(\text{mUCB})] \leq O\Big(\sideset{}{_{i\in\A_+}}{\sum}\frac{\log n}{\min_{\theta\in\Theta_{+,i}}\Gamma_i(\theta,\btheta)}\Big) \leq O\Big(|\A_+|\frac{\log n}{\min_i \min_{\theta\in\Theta_{+,i}}\Gamma_i(\theta,\btheta)}\Big).
\end{align*}
%
This result suggests that \textit{mUCB} tends to discard all the models in $\Theta_+$ from the most optimistic down to the actual model $\btheta$ which, with high-probability, is never discarded. As a result, even if other models are still in $\Theta_t$, the optimal arm of $\btheta$ is pulled until the end. This significantly reduces the set of arms which are actually pulled by \textit{mUCB} and the previous bound only depend on the number of arms in $\A_+$, which is $|\A_+| \leq |\A_*(\Theta)| \leq K$. Furthermore, it is possible to show that for all arms $i$, the minimum gap $\min_{\theta\in\Theta_{+,i}}\Gamma_i(\theta,\btheta)$ is guaranteed to be larger than the arm gap $\Delta_i(\btheta)$ (see Lem.~\ref{l:gaps} in Sec.~\ref{app:proofs.mucb}), thus further improving the performance of \textit{mUCB} w.r.t. \textit{UCB}.
%
%
%
\vspace{-0.1in}
\section{Online Transfer with Unknown Models}\label{s:tucb}
\vspace{-0.1in}
We now consider the case when the set of models is unknown and the regret is cumulated over multiple tasks drawn from $\rho$ (Eq.~\ref{eq:global.regret}). We introduce \textit{tUCB} (transfer-UCB) which transfers estimates of $\Theta$, whose accuracy is improved through episodes using a method-of-moments approach. 
%
%
%
%
\vspace{-0.1in}
\subsection{The transfer-UCB Bandit Algorithm}\label{ss:tucb}
\vspace{-0.05in}
\begin{figure}[ht]
\begin{minipage}[b]{0.4\linewidth}
{\bookboxx{
\begin{algorithmic}
\REQUIRE number of arms $K$, number of models $m$, constant $C(\theta)$.
\STATE \textbf{Initialize} estimated models $\Theta^1=\{\hmu_i^1(\theta)\}_{i,\theta}$, samples $R\in\mathbb{R}^{J\times K\times n}$
\FOR{$j=1,2,\ldots,J$}
\STATE Run $R^j = \text{umUCB}(\Theta^j, n)$
\STATE Run $\Theta^{j+1} = \text{RTP}(R, m, K, j,\delta)$
\ENDFOR
\end{algorithmic} 
}}
\vspace{0.15in}
\caption{The \textit{tUCB} algorithm. }
\label{alg:mucb-mom} 
\par\vspace{-0.2in}
\end{minipage}
\hspace{0.1in}
\begin{minipage}[b]{0.58\linewidth}
\centering
{\bookboxx{
\begin{algorithmic}
\REQUIRE set of models $\Theta^j$, num. steps $n$
\STATE Pull each arm three times
\FOR{$t = 3K+1,\ldots,n$}
\STATE Build $\Theta_t^j = \{\theta: \forall i, |\hmu_i^j(\theta) - \hmu_{i,t}| \leq \eps_{i,t} + \eps^j \}$
\STATE Compute $B_t^j(i; \theta) = \min \big\{(\hmu^{j}_{i}(\theta) + \eps^j), (\hmu_{i,t} + \eps_{i,t})\big\}$
\STATE Compute $\theta_t^j = \arg\max_{\theta\in\Theta_t^j} \max_i B_t^j(i; \theta)$
\STATE Pull arm $I_t = \arg\max_i B_t^j(i; \theta_t^j)$
\STATE Observe sample $R(I_t, T_{i,t}) = x_{I_t}$ and update
\ENDFOR
\RETURN Samples $R$
\end{algorithmic}
}}
\vspace{-0.15in}
\caption{The \textit{umUCB} algorithm.}
\label{f:umUCB}
\par\vspace{-0.2in}
\end{minipage}
\end{figure}

\begin{figure}[ht]
\begin{minipage}[b]{1.0\linewidth}
\centering
{\bookboxx{
\begin{algorithmic}
\REQUIRE samples $R \in \mathbb{R}^{j\times n}$, number of   models $m$ and arms $K$, episode $j$
\STATE Estimate the second and third moment $\widehat M_2$ and $\widehat M_3$ using the reward samples from $R$ (Eq.~\ref{eq:emp.mom})
\STATE Compute $\widehat D \in\mathbb{R}^{m\times m}$ and $\widehat U\in\mathbb{R}^{K \times m}$ ($m$ largest eigenvalues and eigenvectors of $\widehat M_2$ resp.)
\STATE Compute the whitening mapping $\widehat W=\widehat U \widehat D^{-1/2}$ and the tensor $\widehat T=\widehat M_3(\widehat W,\widehat W, \widehat W)$
\STATE Plug $\widehat T$ in Alg.~1 of \citet{AnandkumarRHKT12} and compute eigen-vectors/values $\{\widehat v(\theta)\}$, $\{\widehat \lambda(\theta)\}$
\STATE Compute $\widehat \mu^{j}(\theta)=\widehat \lambda(\theta)(\widehat W^{\mathsf{T}} )^+\widehat v(\theta)$ for all $\theta\in\Theta$
\RETURN $\Theta^{j+1}=\{\widehat \mu^{j}(\theta):\theta\in\Theta\}$
\end{algorithmic}
}}
\vspace{-0.15in}
\caption{The robust tensor power (\textit{RTP}) method.}
\label{f:MoM}
\par\vspace{-0.2in}
\end{minipage}
\end{figure}
Fig.~\ref{alg:mucb-mom} outlines the structure of our online transfer bandit algorithm \textit{tUCB} (transfer-UCB). The algorithm uses two sub-algorithms, the bandit algorithm \textit{umUCB} (\textit{uncertain model-UCB}), whose objective is to minimize the regret at each episode, and \textit{RTP} (\textit{robust tensor power} method) which at each episode $j$ computes an estimate $\{\hmu_i^j(\theta)\}$ of the arm means of all the models. The bandit algorithm \textit{umUCB} in Fig.~\ref{f:umUCB} is an extension of the \textit{mUCB} algorithm. It first computes a set of models $\Theta_t^j$ whose means $\hmu_i(\theta)$ are compatible with the current estimates $\hmu_{i,t}$. However, unlike the case where the exact models are available, here the models themselves are estimated and the uncertainty $\eps^j$ in their means (provided as input to \textit{umUCB}) is taken into account in the definition of $\Theta_t^j$. Once the active set is computed, the algorithm computes an upper-confidence bound on the value of each arm $i$ for each model $\theta$ and returns the best arm for the most optimistic model. Unlike in \textit{mUCB}, due to the uncertainty over the model estimates, a model $\theta$ might have more than one optimal arm, and an upper-confidence bound on the mean of the arms $\hmu_i(\theta)+\eps^j$ is used together with the upper-confidence bound $\hmu_{i,t}+\eps_{i,t}$, which is directly derived from the samples observed so far from arm $i$. This guarantees that the $B$-values are always consistent with the samples generated from the actual model $\btheta^j$. Once \textit{umUCB} terminates, \textit{RTP} (Fig.~\ref{f:MoM}) updates the estimates of the model means $\widehat \mu^j(\theta)=\{\hmu^j_i(\theta)\}_i \in \mathbb{R}^K$ using the  samples obtained from each arm $i$. At the beginning of each task \textit{umUCB} pulls all the arms $3$ times, since \textit{RTP} needs at least $3$ samples from each arm to accurately estimate the $2^{\text{nd}}$ and $3^{\text{rd}}$ moments~\citep{AnandkumarRHKT12}.  More precisely, \textit{RTP} uses all the reward samples generated up to episode $j$ to estimate the $2^{\text{nd}}$ and $3^{\text{rd}}$ moments (see Sec.~\ref{s:preliminaries}) as
%
\begin{equation}
\label{eq:emp.mom}
\widehat M_2=j^{-1}\sideset{}{_{l=1}^j}{ \sum}\overline \mu_{1l}\otimes \overline \mu_{2l},\qquad\text{and}\qquad \widehat M_3= j^{-1}\sideset{}{_{l=1}^j}\sum\overline \mu_{1l}\otimes \overline \mu_{2l}\otimes \overline \mu_{3l},
\end{equation}
%
where the vectors $\overline \mu_{1l},\overline \mu_{2l},\overline \mu_{3l}\in \mathbb{R}^K$ are obtained by dividing the $T_{i,n}^l$ samples observed from arm $i$ in episode $l$ in three batches and taking their average (e.g., $[\overline \mu_{1l}]_i$ is the average of the first $T_{i,n}^l/3$ samples).\footnote{Notice that $1/3([\overline \mu_{1l}]_i + [\overline \mu_{2l}]_i + [\overline \mu_{1l}]_i) = \hmu_{i,n}^l$, the empirical mean of arm $i$ at the end of episode $l$.} 
Since $\overline \mu_{1l},\overline \mu_{2l},\overline \mu_{3l}$  are independent estimates of $\mu(\btheta^l)$, $\widehat M_2$ and $\widehat M_3$ are consistent estimates of the second and third moments $M_2$ and $M_3$. \textit{RTP} relies on the  fact that  the model  means $\mu(\theta)$ can be recovered from the spectral decomposition of the symmetric tensor $T=M_3(W,W,W)$, where $W$ is a whitening matrix  for $M_2$, i.e., $M_2(W,W)=\mathbf{I}^{m\times m}$ (see Sec.~\ref{s:preliminaries} for the definition of the mapping $A(V_1,V_2,V_3)$).  \citet{AnandkumarRHKT12} (Thm. 4.3) have shown that  under some  mild assumption (see later Assumption \ref{asm:ndegen}) the model means $\{\mu(\theta)\}$, can be obtained as
$\mu(\theta)=\lambda(\theta) B v(\theta)$,
where  $(\lambda(\theta),v(\theta))$ is a pair of eigenvector/eigenvalue for the tensor $T$ and $B:=({W^{\mathsf{T}}})^+$.
Thus the \textit{RTP} algorithm  estimates the eigenvectors $\widehat v(\theta)$ and the  eigenvalues $\widehat \lambda(\theta)$,  of  the $m\times m\times m$ tensor $\widehat T:=\widehat M_3(\widehat W, \widehat W, \widehat W)$.\footnote{The matrix $\widehat W\in\mathbb{R}^{K\times m}$ is such that $\widehat M_2(\widehat W,\widehat W)=\mathbf{I}^{m\times m}$, i.e., $\widehat W$ is the whitening matrix of  $\widehat M_2$. In general $\widehat W$ is not unique. Here, we choose $\widehat W=\widehat U \widehat D^{-1/2}$, where $\widehat D\in\mathbb{R}^{m\times m}$ is a diagonal matrix consisting of the  $m$ largest eigenvalues of $\widehat M_2$ and $\widehat U\in\mathbb{R}^{K\times m}$ has the corresponding eigenvectors as its columns.}   Once $\widehat v (\theta)$ and   $\widehat \lambda (\theta)$ are computed, the estimated mean vector $\widehat \mu^j(\theta)$ is obtained 
by the inverse transformation $\widehat \mu^j(\theta)=\widehat \lambda (\theta)\widehat B \widehat v(\theta)$, where $\widehat B$ is the pseudo inverse of $\widehat W^{\mathsf{T}}$\citep[for a detailed description of RTP algorithm see][]{AnandkumarRHKT12}.  
%
%

%
\vspace{-0.05in}
\subsection{Sample Complexity of the Robust Tensor Power Method}\label{ss:mom}
\vspace{-0.05in}
\textit{umUCB} requires as input $\eps^j$, i.e., the uncertainty of the model estimates. Therefore we need finite sample complexity bounds on the accuracy of $\{\hmu_i(\theta)\}$ computed by \textit{RTP}.
The performance of \textit{RTP} is directly affected by the error of the estimates $\widehat{M_2}$ and $\widehat{M_3}$ w.r.t. the true moments.
In Thm.~\ref{thm:mom} we  prove that, as the number of tasks  $j$ grows, this error  rapidly decreases  with the rate of $\sqrt{1/j}$. This result provides us with an upper-bound on the error $\varepsilon^j$ needed for building the confidence intervals in \textit{umUCB}.  
The following definition and assumption are required for our result.
\begin{definition}
Let $\Sigma_{M_2}=\{\sigma_1,\sigma_2,\dots,\sigma_m\}$  be the set of $m$ largest  eigenvalues of the matrix $M_2$. Define $\sigma_{\min}:=\min_{\sigma\in\Sigma_{M_2}}\sigma$,  $\sigma_{\max}:=\min_{\sigma\in\Sigma_{M_2}}\sigma$ and  $\lambda_{\max}:=\max_{\theta}\lambda(\theta)$.  Define the minimum gap between the distinct eigenvalues of $M_2$ as 
$\Gamma_{\sigma}:=\min_{\sigma_i\neq\sigma_l }(|\sigma_i-\sigma_l|)$.
 %
\end{definition}
%
%
\begin{asm}
\label{asm:ndegen}
The mean vectors $\{\mu(\theta)\}_\theta$ are linear independent  and $\rho(\theta)>0$ for all $\theta\in\Theta$. 
\end{asm}
We now state our main result which is in the form of a high probability bound on  the estimation error of mean reward vector of every model $\theta\in\Theta$.
\begin{theorem}\label{thm:mom}
Pick $\delta\in(0,1)$. Let $C(\Theta):=C_3\lambda_{\max}\sqrt{\frac{\sigma_{\max}} {\sigma_{\min}^{3}}} \left(\frac{\sigma_{\max}}{\Gamma_{\sigma}}+\frac{1}{\sigma_{\min}}+\frac{1}{\sigma_{\max}}\right)$, where $C_3>0$ is a universal constant. Then under Assumption \ref{asm:ndegen} there exist constants  $C_4>0$  and a permutation $\pi$ on $\Theta$ such that after $j$ tasks
%
\begin{equation*}\label{eq:mom.bound}
\begin{aligned}
\max_{\theta}\|\mu(\theta)-\widehat \mu^j(\pi(\theta))\|\leq C(\Theta)K^{2.5}m^2\sqrt{\frac{\log(K/\delta)}{j}},
\end{aligned}
\end{equation*}
%
w.p. $ 1-\delta$, given that 
%
%
\begin{equation}
\label{eq:bound.min.samp} 
j\geq \tfrac{C_4 m^5 K^6 \log(K/\delta)}{\min(\sigma_{\min},\Gamma_{\sigma})^2\sigma^3_{\min}\lambda^2_{\min}}.
\end{equation}
%
%
\end{theorem}
\begin{remark}{(comparison with the previous bounds)}
This bound improves on the previous bounds of~\citet{AnandkumarHK12,AnandkumarFHKL12} moving from a dependency on the number of models of order $O(m^5)$ to a milder quadratic dependency on $m$.\footnote{Note that the improvement is mainly due to accuracy of the orthogonal tensor decomposition  obtained via the tensor power method relative
to the previously cited works. This is a direct consequence  of  the perturbation bound of \citet[Thm. 5.1]{AnandkumarRHKT12}, which is at the core of our sample complexity bound.}~\footnote{The result of \citet{AnandkumarFHKL12} has the explicit dependency of order $m^3$ on the number of model as well as implicit dependency of order $m^2$ through the parameter $\alpha_0$.}  Although the dependency on $\sigma_{\min}$ is a bit worse in our bounds in comparison to those of  \citet{AnandkumarHK12,AnandkumarFHKL12}, here we have the advantage that  there is no  dependency on the smallest singular value of the matrix $\{\mu(\theta):\theta\in\Theta\}$, whereas those results  scale polynomially with this factor. 
\end{remark}

\begin{remark}{(computation of $C(\Theta)$)}
As illustrated in Fig.~\ref{f:umUCB}, \textit{umUCB} relies on the estimates $\widehat \mu^j(\theta)$ and on their accuracy $\eps^j$. Although the bound reported in Thm.~\ref{thm:mom} provides an upper confidence bound on the error of the estimates, it contains terms which are not computable in general (e.g., $\sigma_{\min}$). In practice, $C(\Theta)$ should be considered as a parameter of the algorithm.\footnote{One  may also estimate the constant $C(\Theta)$ in an online fashion using doubling trick~\citep{audibert201216}.} This is not dissimilar from to the parameter usually introduced in the definition of $\eps_{i,t}$ in front of the square-root term in \textit{UCB}. 
\end{remark}
%
%
%

\vspace{-0.05in}
\subsection{Regret Analysis of \textit{umUCB}}\label{ss:umucb.regret}
\vspace{-0.05in}

We now analyze the regret of \textit{umUCB} when an estimated set of models $\Theta^j$ is provided as input.
At episode $j$, for each model $\theta$ we define the set of non-dominated arms (i.e., potentially optimal arms) as $\A_*^j(\theta) = \{i\in\A: \nexists i', \hmu_i^j(\theta)+\eps^j < \hmu_{i'}^j(\theta)-\eps^j\}$. Among the non-dominated arms, when the actual model is $\btheta^j$, the set of optimistic arms is $\A_+^j(\theta; \btheta^j) = \{i\in\A_*^j(\theta): \hmu_i^j(\theta)+\eps^j \geq \mu^*(\btheta^j) \}$. As a result, the set of optimistic models is $\Theta_+^j(\btheta^j) = \{\theta\in\Theta: \A_+^j(\theta; \btheta^j)\neq \emptyset\}$. In some cases, because of the uncertainty in the model estimates, unlike in \textit{mUCB}, not all the models $\theta\neq \btheta^j$ can be discarded, not even at the end of a very long episode. Among the optimistic models, the set of models that cannot be discarded is defined as $\tTheta_+^j(\btheta^j) = \{\theta\in\Theta_+^j(\btheta^j): \forall i\in\A_+^j(\theta;\btheta^j), |\hmu_i^j(\theta)-\mu_i(\btheta^j)| \leq \eps^j\}$.
%
%
Finally, when we want to apply the previous definitions to a set of models $\Theta'$ instead of single model we have, e.g.,  $\A_*^j(\Theta'; \btheta^j) = \bigcup_{\theta\in\Theta'} \A_*^j(\theta; \btheta^j)$. 

The proof of the following results are available in Sec.~\ref{app:proofs.umucb} of the supplementary material, here we only report the number of pulls, and the corresponding regret bound.

\begin{corollary}\label{l:umucb.pulls}
If at episode $j$ \textit{umUCB} is run with $\eps_{i,t}$ as in Eq.~\ref{eq:eps.def} and $\eps^j$ as in Eq.~\ref{eq:mom.bound} with a parameter $\delta'=\delta/2K$,
then for any arm $i\in\A$, $i\neq i_*(\btheta^j)$ is pulled $T_{i,n}$ times such that
\begin{small}
\begin{align*}
\begin{dcases} 
T_{i,n} \leq \min\bigg\{ \frac{2\log \big(2mKn^2/\delta\big)}{\Delta_i(\btheta^j)^2} , \frac{\log \big(2mKn^2/\delta\big)}{2\sideset{}{_{\theta\in\Theta_{i,+}^j(\btheta^j)}}{\min\limits}\widehat\Gamma_i(\theta; \btheta^j)^2} \bigg\}+1 &\mbox{if } i\in\A_1^j \\
T_{i,n} \leq  2\log \big(2mKn^2/\delta\big)/(\Delta_i(\btheta^j)^2)+1 &\mbox{if } i\in\A_2^j \\
T_{i,n} = 0 &\mbox{otherwise}
\end{dcases}
\end{align*}
\end{small}
w.p. $1-\delta$, where $\Theta_{i,+}^j(\btheta^j) = \{\theta\in\Theta_+^j(\btheta^j)\!:\! i\in\A_+(\theta; \btheta^j)\}$ is the set of models for which $i$ is among theirs optimistic non-dominated arms, $\widehat\Gamma_i(\theta; \btheta^j) = \Gamma_i(\theta, \btheta^j)/2 - \eps^j$, $\A^j_1 = \A_+^j(\Theta_+^j(\btheta^j); \btheta^j) - \A_+^j(\tTheta_+^j(\btheta^j); \btheta^j)$ (i.e., set of arms only proposed by models that can be discarded), and $\A_2^j = \A_+^j(\tTheta_+^j(\btheta^j); \btheta^j)$ (i.e., set of arms only proposed by models that cannot be discarded).
\end{corollary}

The previous corollary states that arms which cannot be optimal for any optimistic model (i.e., the optimistic non-dominated arms) are never pulled by \textit{umUCB}, which focuses only on arms in $i\in\A_+^j(\Theta_+^j(\btheta^j); \btheta^j)$. Among these arms, those that may help to remove a model from the active set (i.e., $i\in\A^j_1$) are potentially pulled less than \textit{UCB}, while the remaining arms, which are optimal for the models that cannot be discarded (i.e., $i\in\A^j_2$), are simply pulled according to a \textit{UCB} strategy.
Similar to \textit{mUCB}, \textit{umUCB} first pulls the arms that are more \textit{optimistic} until either the active set $\Theta_t^j$ changes or they are no longer optimistic (because of the evidence from the actual samples). We are now ready to derive the per-episode regret of \textit{umUCB}. 
\begin{theorem}\label{thm:umucb.regret}
If \textit{umUCB} is run for $n$ steps on the set of models $\Theta^j$ estimated by \textit{RTP} after $j$ episodes with $\delta=1/n$, and the actual model is $\btheta^j$, then its expected regret (w.r.t. the random realization in episode $j$ and conditional on $\btheta^j$) is
\begin{small}
\begin{align*}
\E[\R_n^j] \leq K \!+\!\!\sideset{}{ _{i\in \A_1^j}}{\sum}  \min\bigg\{ \frac{2\log \big(2mKn^3\big)}{\Delta_i(\btheta^j)^2} , \frac{\log \big(2mKn^3\big)}{2\sideset{}{_{\theta\in\Theta_{i,+}^j(\btheta^j)}}{\min}\widehat\Gamma_i(\theta; \btheta^j)^2} \bigg\}\Delta_i(\btheta^j) + \sideset{}{_{i\in\A_2^j}}\sum \frac{2\log \big(2mKn^3\big)}{\Delta_i(\btheta^j)}.
\end{align*}
\end{small}
\end{theorem}
%

\textbf{Remark (negative transfer).} The transfer of knowledge introduces a bias in the learning process which is often beneficial.
Nonetheless, in many cases transfer may result in a bias towards wrong solutions and a worse learning performance, a phenomenon often referred to as \textit{negative transfer}. The first interesting aspect of the previous theorem is that \textit{umUCB} is guaranteed to never perform worse than \textit{UCB} itself. This implies that \textit{tUCB} never suffers from negative transfer, even when the set $\Theta^j$ contains highly uncertain models and might bias \textit{umUCB} to pull suboptimal arms.

\textbf{Remark (improvement over \textit{UCB}).} In Sec.~\ref{s:mucb} we showed that \textit{mUCB} exploits the knowledge of $\Theta$ to focus on a restricted set of arms which are pulled less than \textit{UCB}. In \textit{umUCB} this improvement is not as clear, since the models in $\Theta$ are not known but are estimated online through episodes. Yet, similar to \textit{mUCB}, \textit{umUCB} has the two main sources of potential improvement w.r.t. to \textit{UCB}. As illustrated by the regret bound in Thm.~\ref{thm:umucb.regret}, \textit{umUCB} focuses on arms in $\A_1^j \cup \A_2^j$ which is potentially a smaller set than $\A$. Furthermore, the number of pulls to arms in $\A_1^j$ is smaller than for \textit{UCB} whenever the estimated model gap $\widehat\Gamma_i(\theta; \btheta^j)$ is bigger than $\Delta_i(\btheta^j)$. Eventually, \textit{umUCB} reaches the same performance (and improvement over \textit{UCB}) as \textit{mUCB} when $j$ is big enough. In fact, the set of optimistic models reduces to the one used in \textit{mUCB} (i.e., $\Theta_+^j(\btheta^j) \equiv \Theta_+(\btheta^j)$) and all the optimistic models have only optimal arms (i.e., for any $\theta\in\Theta_+$ the set of non-dominated optimistic arms is $\A_+(\theta; \btheta^j) = \{i_*(\theta)\}$), which corresponds to $\A_1^j \equiv \A_*(\Theta_+(\btheta^j))$ and $\A_2^j \equiv \{i_*(\btheta^j)\}$, which matches the condition of \textit{mUCB}. For instance, for any model $\theta$, to have $\A_*(\theta) = \{i_*(\theta)\}$ we need for any arm $i\neq i_*(\theta)$ that $\hmu_i^j(\theta) + \eps^j \leq \hmu_{i_*(\theta)}^j(\theta) - \eps^j$. As a result $j \geq 2C(\Theta) / \min\limits_{\btheta\in\Theta}\min\limits_{\theta\in\Theta_+(\btheta)}\min_i \Delta_i(\theta)^2 + 1$ episodes are needed
in order for all the optimistic models to have only one optimal arm independently from the actual identity of the model $\btheta^j$. Although this condition may seem restrictive, in practice \textit{umUCB} starts improving over \textit{UCB} much earlier, as illustrated in the numerical simulation in Sec.~\ref{s:experiments}. 


\vspace{-0.05in}
\subsection{Regret Analysis of \textit{tUCB}}\label{ss:tucb.regret}
\vspace{-0.05in}

Given the previous results, we derive the bound on the cumulative regret over $J$ episodes (Eq.~\ref{eq:global.regret}).

\begin{theorem}\label{thm:final.bound}
If \textit{tUCB} is run over $J$ episodes of $n$ steps in which the tasks $\btheta^j$ are drawn from a fixed distribution $\rho$ over a set of models $\Theta$, then its cumulative regret is
\begin{small}
\begin{align*}
\R_J \leq JK  + \sum_{j=1}^J \sum_{i\in \A_1^j}  \min\bigg\{ \frac{2\log \big(2mKn^2/\delta\big)}{\Delta_i(\btheta^j)^2} , \frac{\log \big(2mKn^2/\delta\big)}{2\min\limits_{\theta\in\Theta_{i,+}^j(\btheta^j)}\widehat\Gamma_i^j(\theta; \btheta^j)^2} \bigg\}\Delta_i(\btheta^j) + \sum_{j=1}^J\sum_{i\in\A_2^j} \frac{2\log \big(2mKn^2/\delta\big)}{\Delta_i(\btheta^j)},
\end{align*}
\end{small}
w.p. $1-\delta$ w.r.t. the randomization over tasks and the realizations of the arms in each episode.
\end{theorem}

This result immediately follows from Thm.~\ref{thm:umucb.regret} and it shows a linear dependency on the number of episodes $J$. This dependency is the price to pay for not knowing the identity of the current task $\btheta^j$. If the task was revealed at the beginning of the task, a bandit algorithm could simply cluster all the samples coming from the same task and incur a much smaller cumulative regret with a logarithmic dependency on episodes and steps, i.e., $\log(nJ)$. Nonetheless, as discussed in the previous section, the cumulative regret of \textit{tUCB} is never worse than for \textit{UCB} and as the number of tasks increases it approaches the performance of \textit{mUCB}, which fully exploits the prior knowledge of $\Theta$.
   

\vspace{-0.1in}
\section{Numerical Simulations}\label{s:experiments}
\vspace{-0.1in}

\begin{figure}[t]
\begin{center}
\begin{minipage}[b]{0.5\linewidth}
\centering
\includegraphics[trim=5.5cm 0cm 1cm 0cm, clip=true,height=0.19\textheight,width=1.05\textwidth]{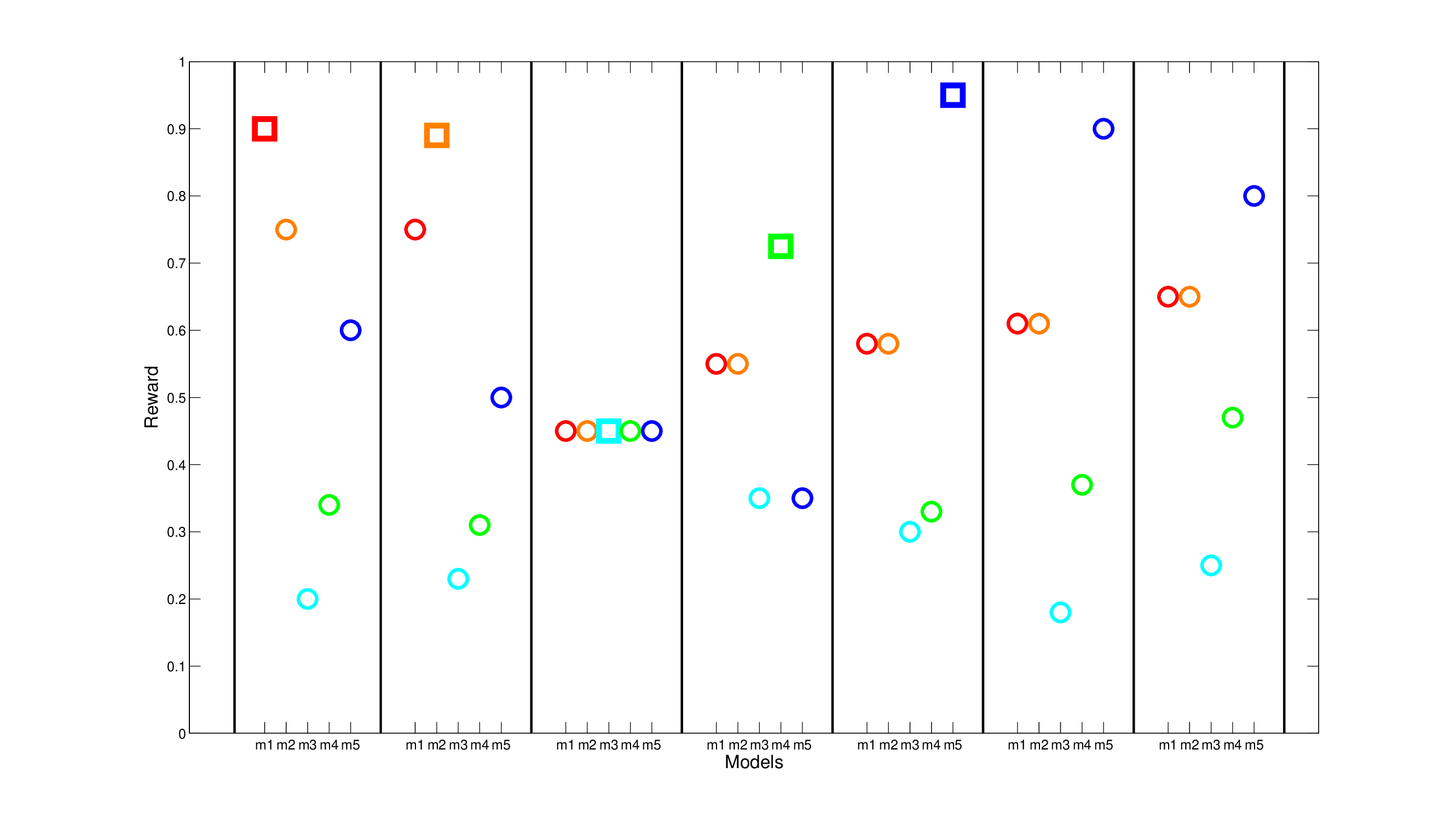}
\par\vspace{-0.1in}
\caption{Set of models $\Theta$.}
\label{f:models}
\par\vspace{-0.15in}
\end{minipage}
\begin{minipage}[b]{0.4\linewidth}
\centering
\includegraphics[width=1.1\textwidth]{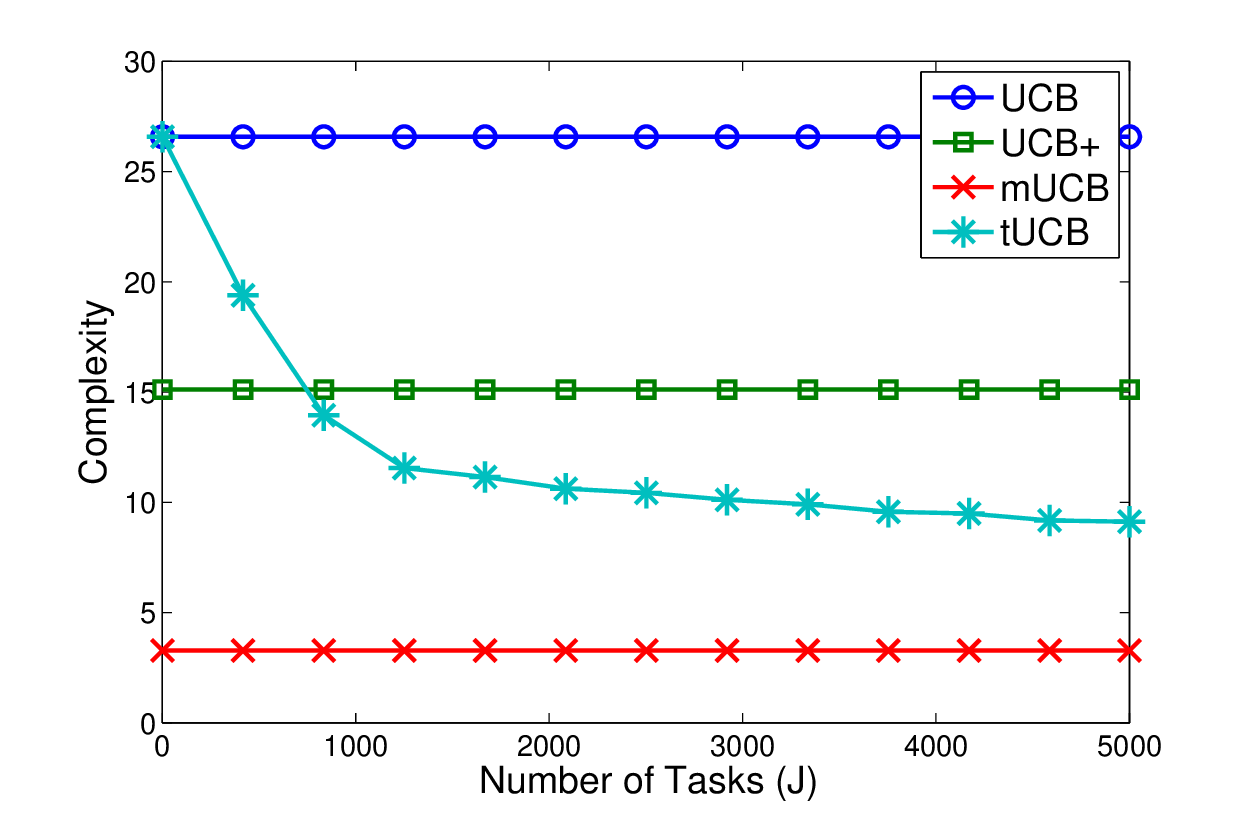}
\par\vspace{-0.01in}
\caption{Complexity over tasks.}
\label{f:complexity}
\par\vspace{-0.15in}
\end{minipage}
\end{center}
\end{figure}

In this section we report preliminary results of \textit{tUCB} on synthetic data. The objective is to illustrate and support the previous theoretical findings. We define a set $\Theta$ of $m=5$ MAB problems with $K=7$ arms each, whose means $\{\mu_i(\theta)\}_{i,\theta}$ are reported in Fig.~\ref{f:models} (see Sect.~\ref{app:plus.experiment} in the supplementary material for the actual values), where each model has a different color and squares correspond to optimal arms (e.g., arm $2$ is optimal for model $\theta_2$). This set of models is chosen to be challenging and illustrate some interesting cases useful to understand the functioning of the algorithm.\footnote{Notice that although $\Theta$ satisfies Assumption~\ref{asm:ndegen}, the smallest singular value $\sigma_{\min} = 0.0039$ and $\Gamma_\sigma = 0.0038$, thus making the estimation of the models difficult.} Models $\theta_1$ and $\theta_2$ only differ in their optimal arms and this makes it difficult to distinguish them. For arm 3 (which is optimal for model $\theta_3$ and thus potentially selected by \textit{mUCB}), all the models share exactly the same mean value. This implies that no model can be discarded by pulling it. Although this might suggest that \textit{mUCB} gets stuck in pulling arm 3, we showed in Thm.~\ref{thm:mucb.regret} that this is not the case. Models $\theta_1$ and $\theta_5$ are challenging for \textit{UCB} since they have small minimum gap. Only 5 out of the 7 arms are actually optimal for a model in $\Theta$. Thus, we also report the performance of \textit{UCB+} which, under the assumption that $\Theta$ is known, immediately discards all the arms which are not optimal ($i\notin\A^*$) and performs \textit{UCB} on the remaining arms. The model distribution is uniform, i.e., $\rho(\theta) = 1/m$.

\begin{figure}[t]
\begin{center}
\begin{minipage}[b]{0.48\linewidth}
\vspace{0pt}
\centering
\includegraphics[trim=0.5cm 0cm 1cm 0cm, clip=true,width=0.85\textwidth]{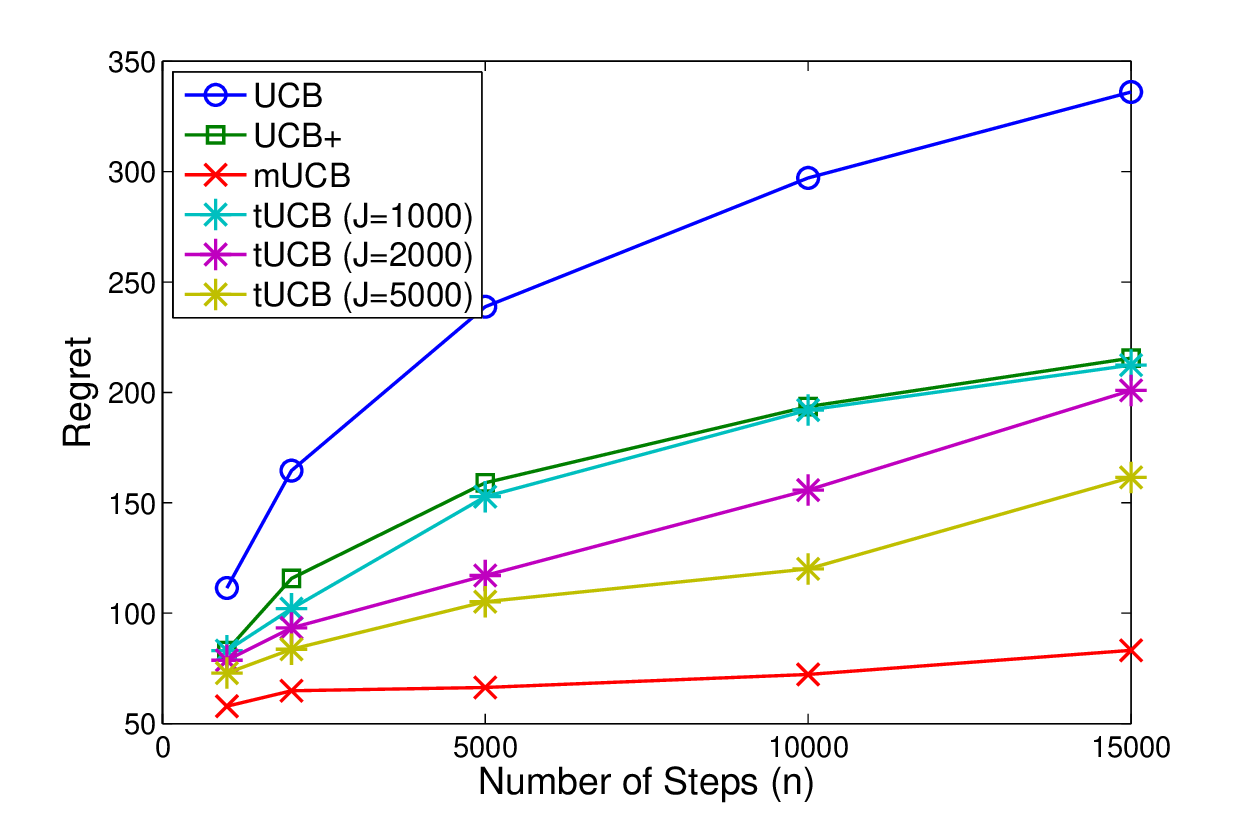}
\par\vspace{-0.1in}
\caption{Regret of \textit{UCB}, \textit{UCB+}, \textit{mUCB}, and \textit{tUCB} (avg. over episodes) vs episode length.}
\label{f:comparison}
\par\vspace{-0.2in}
\end{minipage}
\hspace{0.05in}
\begin{minipage}[b]{0.48\linewidth}
\vspace{0pt}
\centering
\includegraphics[trim=0.5cm 0cm 1cm 0cm, clip=true,width=0.85\textwidth]{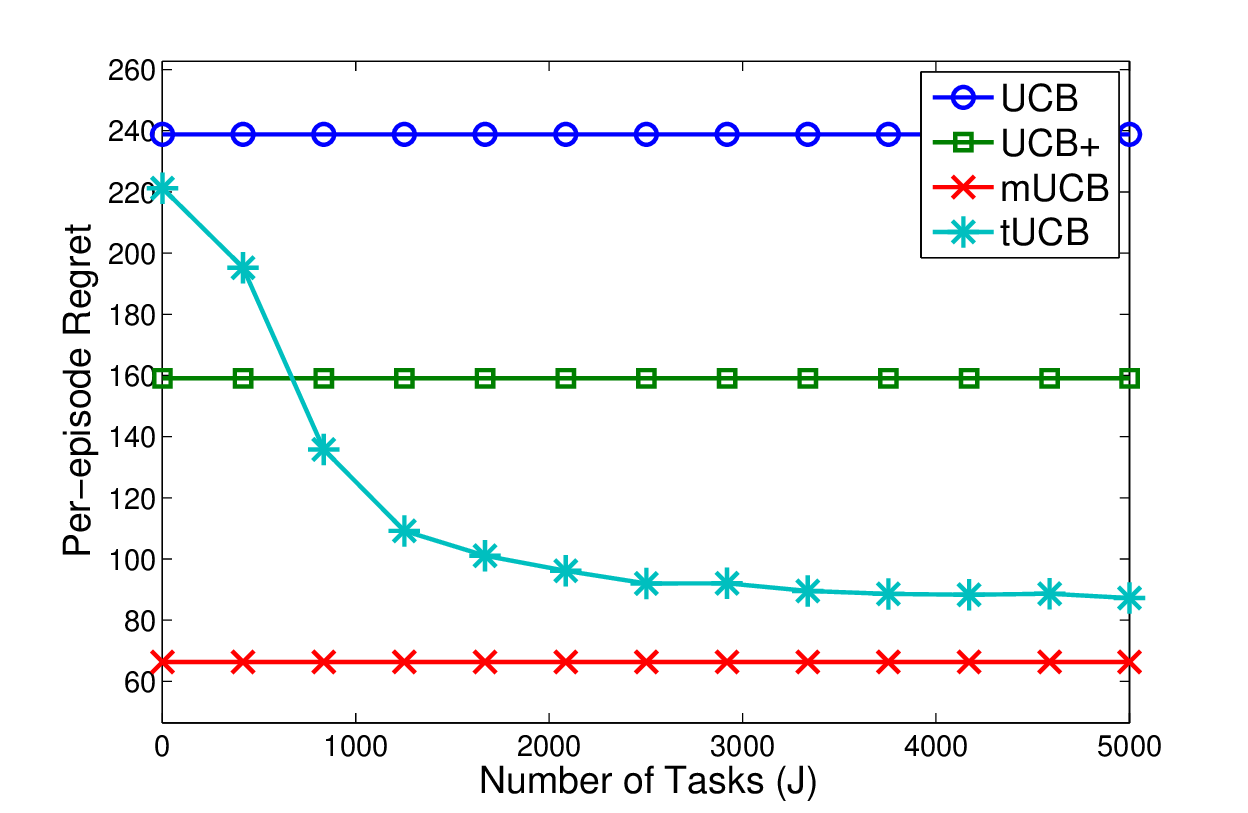}
\par\vspace{-0.0in}
\caption{Per-episode regret of \textit{tUCB}.}
\label{f:tucb}
\par\vspace{-0.2in}
\end{minipage}
\end{center}
\end{figure}

Before discussing the transfer results, we compare \textit{UCB}, \textit{UCB+}, and \textit{mUCB}, to illustrate the advantage of the prior knowledge of $\Theta$ w.r.t. \textit{UCB}. Fig.~\ref{f:comparison} reports the per-episode regret of the three algorithms for episodes of different length $n$ (the performance of \textit{tUCB} is discussed later). The results are averaged over all the models in $\Theta$ and over $200$ runs each. All the algorithms use the same confidence bound $\eps_{i,t}$. The performance of \textit{mUCB} is significantly better than both \textit{UCB}, and \textit{UCB+}, thus showing that \textit{mUCB} makes an efficient use of the prior of knowledge of $\Theta$. Furthermore, in Fig.~\ref{f:complexity} the horizontal lines correspond to the value of the regret bounds up to the $n$ dependent terms and constants\footnote{For instance, for \textit{UCB} we compute $\sum_i 1/\Delta_i$.} for the different models in $\Theta$ averaged w.r.t. $\rho$ for the three algorithms (the actual values for the different models are in the supplementary material). These values show that the improvement observed in practice is accurately predicated by the upper-bounds derived in Thm.~\ref{thm:mucb.regret}. 

We now move to analyze the performance of \textit{tUCB}. In Fig.~\ref{f:tucb} we show how the per-episode regret changes through episodes for a transfer problem with $J=5000$ tasks of length $n=5000$. In \textit{tUCB} we used $\eps^j$ as in Eq.\ref{eq:mom.bound} with $C(\Theta)=2$.
As discussed in Thm.~\ref{thm:umucb.regret}, \textit{UCB} and \textit{mUCB} define the boundaries of the performance of \textit{tUCB}. In fact, at the beginning \textit{tUCB} selects arms according to a \textit{UCB} strategy, since no prior information about the models $\Theta$ is available. On the other hand, as more tasks are observed, \textit{tUCB} is able to transfer the knowledge acquired through episodes and build an increasingly accurate estimate of the models, thus approaching the behavior of \textit{mUCB}. This is also confirmed by Fig.~\ref{f:complexity} where we show how the complexity of \textit{tUCB} changes through episodes. In both cases (regret and complexity) we see that \textit{tUCB} does not reach the same performance of \textit{mUCB}. This is due to the fact that some models have relatively small gaps and thus the number of episodes to have an accurate enough estimate of the models to reach the performance of \textit{mUCB} is much larger than 5000 (see also the Remarks of Thm.~\ref{thm:umucb.regret}). 
Since the final objective is to achieve a small global regret (Eq.~\ref{eq:global.regret}), in Fig.~\ref{f:comparison} we report the cumulative regret averaged over the total number of tasks ($J$) for different values of $J$ and $n$. Again, this graph shows that \textit{tUCB} outperforms \textit{UCB} and that it tends to approach the performance of \textit{mUCB} as $J$ increases, for any value of $n$.


\vspace{-0.1in}
\section{Conclusions and Open Questions}\label{s:conclusions}
\vspace{-0.1in}

In this paper we introduce the transfer problem in the multi-armed bandit framework when a tasks are drawn from a finite set of bandit problems. We first introduced the bandit algorithm \textit{mUCB} and we showed that it is able to fully exploit the prior knowledge on the set of bandit problems $\Theta$ and reduce the regret w.r.t. \textit{UCB}. When the set of models is unknown we define a method-of-moments variant (\textit{RTP}) which consistently estimates the means of the models in $\Theta$ from the samples collected through episodes. This knowledge is then transferred to \textit{umUCB} which never performs worse than \textit{UCB} and tends to approach the performance of \textit{mUCB}. For these algorithms we derive regret and sample complexity bounds, and we show preliminary numerical simulations. To the best of our knowledge, this is the first work studying the problem of transfer in multi-armed bandit and it opens a series of interesting questions.


\textit{Optimality of \textit{mUCB}}.
In some cases, \textit{mUCB} may miss the opportunity to explore arms that could be useful in discarding models. For instance, an arm $i\notin \A_*(\Theta)$ may correspond to very large gaps $\Gamma_i(\theta,\btheta)$ and few pulls to it, although leading to large regret, may be enough to discard many models, thus guaranteeing a very small regret in the following. This observation rises the question whether the \textit{optimistic} approach in this case still guarantees an optimal tradeoff between exploration and exploitation. Since the focus of this paper is on transfer and \textit{mUCB} is already guaranteed to perform better than \textit{UCB}, we left this question for future work.

\textit{Optimality of \textit{tUCB}}. At each episode, \textit{tUCB} transfers the knowledge about $\Theta$ acquired from previous tasks to achieve a small per-episode regret using \textit{umUCB}. Although this strategy guarantees that the per-episode regret of \textit{tUCB} is never worse than \text{UCB}, it may not be the optimal strategy in terms of the cumulative regret through episodes. In fact, if $J$ is large, it could be preferable to run a \textit{model identification} algorithm instead of \textit{umUCB} in earlier episodes so as to improve the quality of the estimates $\hmu_i(\theta)$. Although such an algorithm would incur a much larger regret in earlier tasks (up to linear), it could approach the performance of \textit{mUCB} in later episodes much faster than done by \textit{tUCB}. This trade-off between \textit{identification} of the models and \textit{transfer} of knowledge resembles the exploration-exploitation trade-off in the single-task problem and it may suggest that different algorithms than \textit{tUCB} are possible.


\bibliographystyle{apalike}
\bibliography{transfer}	


\newpage
\appendix


\section{Table of Notation}

\begin{center}
\begin{small}
\begin{tabular}{|c|c|}
\hline
\textit{Symbol} & \textit{Explanation} \\
\hline
$\A$ & Set of arms \\
$\Theta$& Set of models\\
$K$ & Number of arms \\
$m$& Number of models\\
$J$ & Number of episodes\\
$n$& Number of steps per episode\\
$t$&Time step\\
$\btheta$& Current model\\
$\Theta_t$&Active set of models at time $t$\\ 
$\nu_i$ & Distribution of  arm $i$\\
$\mu_i(\theta)$& Mean of arm $i$ for model $\theta$\\
$\mu(\theta)$& Vector of means of model $\theta$\\
$\hmu_{i,t}$&Estimate of $\mu_i(\btheta)$ at time $t$\\ 
$\hmu_i^j(\theta)$ &  Estimate of $\mu_i(\theta)$ by RTP for model $\theta$ and arm $i$ at episode $j$\\ 
$\widehat \mu^j(\theta)$&Estimate of $\mu(\theta)$ by RTP for model $\theta$  at episode $j$\\ 
$\Theta^j$&Estimated model  of RTP after $j$  episode\\  
$\eps^j$& Uncertainty of the estimated model by RTP at episode $j$\\
$\eps_{i,t}$& Model uncertainty at time $t$\\
$\delta$& Probability of failure\\
$i_*(\theta)$& Best arm of  model $\theta$\\
$\mu_*(\theta)$&Optimal value of model $\theta$\\  
 $\Delta_i(\theta)$&Arm gap of an arm $i$ for a model $\theta$\\
$\Gamma_i(\theta, \theta')$& Model gap for an arm $i$ between two models $\theta$ and $\theta'$\\
$M_2$ & $2^{\text{nd}}$-order moment\\
$M_3$& $3^{\text{rd}}$-order moment\\
$\widehat M_2$ & Empirical $2^{\text{nd}}$-order moment\\
$\widehat M_3$& Empirical $3^{\text{rd}}$-order moment\\
$\|\cdot\|$ &Euclidean norm\\
$\|\cdot\|_F$&  Frobenius norm\\
$\|\cdot\|_{\max}$& Matrix max-norm\\
$\R_J$&Pseudo-regret\\
$T_{i,n}^j$ & The number of pulls to arm $i$ after $n$ steps of episode $j$\\
$\A_*(\Theta')$&Set of arms which are optimal for at least a model in a set $\Theta'$\\ 
$\Theta(\A')$&Set of models for which the arms in $\A'$ are optimal\\
$\Theta_+$&Set of optimistic models for a given model $\btheta$\\ 
$\A_+ $&Set of optimal arms corresponds to $\Theta_+$\\
$W$&Whitening  matrix of $M_2$ \\
$\widehat W$& Empirical whitening matrix\\
$T$ & $M_2$ under the linear transformation $W$\\
$\widehat T$&  $\widehat M_2$ under the linear transformation $\widehat W$\\
$D$&  Diagonal matrix consisting of the  $m$ largest eigenvalues of $ M_2$\\
$\widehat  D$&  Diagonal matrix consisting of the  $m$ largest eigenvalues of $\widehat M_2$\\
$U$&  $K\times m$ matrix with the corresponding eigenvectors of $D$ as its columns\\
$\widehat U$&  $K\times m$ matrix with the corresponding eigenvectors of $\widehat  D$ as its columns\\
$\lambda(\theta)$ &Eigenvalue of $ T$ associated with $\theta$\\  
$ v(\theta)$ & Eigenvector of $ T$ associated with $\theta$ \\
$\widehat \lambda(\theta)$& Eigenvalue of $\widehat T$ associated with $\theta$\\  
$\widehat v(\theta)$ & Eigenvector of $\widehat T$ associated with $\theta$ \\
$\Sigma_{M_2}$ &Set of $m$ largest  eigenvalues of the matrix $M_2$\\
 $\sigma_{\min}$&Minimum   eigenvalue of $M_2$ among the $m$-largest\\
 $\sigma_{\max}$&Maximum  eigenvalue of $M_2$\\
 $\lambda_{\max}$&Maximum eigenvalue of $T$\\
 $\Gamma_{\sigma}$&Minimum gap between the eigenvalues of $M_2$\\
 $C(\Theta)$&$O\left(\lambda_{\max}\sqrt{\frac{\sigma_{\max}} {\sigma_{\min}^{3}}} \left(\frac{\sigma_{\max}}{\Gamma_{\sigma}}+\frac{1}{\sigma_{\min}}+\frac{1}{\sigma_{\max}}\right)\right)$\\
 $\pi(\theta)$&Permutation on $\theta$\\   
 $\A_*^j(\theta)$&Set of non-dominated arms for model $\theta$ at episode $j$\\
 $\tTheta_+^j$&Set of models that cannot be discarded at episode $j$\\
 $\Theta_{i,+}^j$& Set of models for which $i$ is among the optimistic non-dominated arms at episode $j$\\
\hline
\end{tabular}
\end{small}
\end{center}


\section{Proofs of Section~\ref{s:mucb}}\label{app:proofs.mucb}

\begin{lemma}\label{lem:suboptimal.arms}
\textit{mUCB} never pulls arms which are not optimal for at least one model, that is $\forall i\notin\A_*(\Theta)$, $T_{i,n}=0$ with probability 1. Notice also that $|\A_*(\Theta)|\leq |\Theta|$.
\end{lemma}

\begin{lemma}\label{l:mucb.high.prob}
The actual model $\btheta$ is never discarded with high-probability. Formally, the event $\calE = \{\forall t=1,\ldots,n, \bar\theta\in\Theta_t\}$ holds with probability $\Prob[\calE] \geq 1-\delta$ if
\begin{align*}
\eps_{i,t} = \sqrt{\frac{1}{2T_{i,t-1}} \log \bigg(\frac{mn^2}{\delta}\bigg)},
\end{align*}
where $T_{i,t-1}$ is the number of pulls to arm $i$ at the beginning of step $t$ and $m = |\Theta|$.
\end{lemma}

In the previous lemma we implicitly assumed that $|\Theta|=m\leq K$. In general, the best choice in the definition of $\eps_{i,t}$ has a logarithmic factor with $\min\{|\Theta|,K\}$.

\begin{lemma}\label{lem:hp.suboptimal.arms}
On event $\calE$, all the arms $i\notin \A_*(\Theta_+)$, i.e., arms which are not optimal for any of the optimistic models, are never pulled, i.e., $T_{i,n} = 0$ with probability $1-\delta$.
\end{lemma}

The previous lemma suggests that \textit{mUCB} tends to discard all the models in $\Theta_+$ from the most optimistic down to the actual model $\btheta$ which, on event $\calE$, is never discarded. As a result, even if other models are still in $\Theta_t$, the optimal arm of $\btheta$ is pulled until the end.
Finally, we show that the model gaps of interest (see Thm.~\ref{thm:mucb.regret}) are always bigger than the arm gaps.

\begin{lemma}\label{l:gaps}
For any model $\theta\in\Theta_+$, $\Gamma_{i_*(\theta)}(\theta,\btheta) \geq \Delta_{i_*(\theta)}(\btheta)$.
\end{lemma}

\begin{proof}[Proof of Lem.~\ref{lem:suboptimal.arms}]
From the definition of the algorithm we notice that $I_t$ can only correspond to the optimal arm $i_*$ of one model in the set $\Theta_t$. Since $\Theta_t$ can at most contain all the models in $\Theta$, all the arms which are not optimal are never pulled.
\end{proof}

\begin{proof}[Proof of Lem.~\ref{l:mucb.high.prob}]
We compute the probability of the complementary event $\calE^C$, that is that event on which there exist at least one step $t=1,\ldots,n$ where the true model $\bar\theta$ is not in $\Theta_t$. By definition of $\Theta_t$, we have that
\begin{align*}
\calE = \{\forall t, \bar\theta\in\Theta_t\} = \{\forall t, \forall i\in\A, |\mu_i-\hmu_{i,t}|\leq \eps_{i,t}\},
\end{align*}
then
\begin{align*}
\Prob[\calE^C] &= \Prob[\exists t, i, |\mu_i-\hmu_{i,t}|\geq \eps_{i,t}] \leq \sum_{t=1}^n \sum_{i\in\A} \Prob[|\mu_i-\hmu_{i,t}|\geq \eps_{i,t}] = \sum_{t=1}^n \sum_{i\in\A^*(\Theta)} \Prob[|\mu_i-\hmu_{i,t}|\geq \eps_{i,t}]
\end{align*}
where the upper-bounding is a simple union bound and the last passage comes from the fact that the probability for the arms which are never pulled is always 0 according to Lem.~\ref{lem:suboptimal.arms}. At time $t$, $\hmu_{i,t}$ is the empirical average of the $T_{i,t-1}$ samples observed from arm $i$ up to the beginning of round $t$. We define the confidence $\eps_{i,t}$ as
\begin{align*}
\eps_{i,t} = \sqrt{\frac{1}{2T_{i,t-1}} \log \bigg(\frac{|\Theta|n^\alpha}{\delta}\bigg)},
\end{align*}
where $\delta\in(0,1)$ and $\alpha$ is a constant chosen later. Since $T_{i,t-1}$ is a random variable, we need to take an additional union bound over $T_{i,t-1}=1,\ldots,t-1$ thus obtaining
\begin{align*}
\Prob[\calE^C] &\leq \sum_{t=1}^n \sum_{i\in\A^*(\Theta)} \sum_{T_{i,t-1}=1}^{t-1} \Prob[|\mu_i-\hmu_{i,t}|\geq \eps_{i,t}] \\
&\leq \sum_{t=1}^n \sum_{i\in\A^*(\Theta)} \sum_{T_{i,t-1}=1}^{t-1} 2\exp\big(-2T_{i,t-1}\eps_{i,t}^2\big) \leq n(n-1) \frac{|\A^*(\Theta)|\delta}{|\Theta|n^\alpha}.
\end{align*}
Since $|\A^*(\Theta)| < |\Theta|$ (see Lem.~\ref{lem:suboptimal.arms}) and by taking $\alpha=2$ we finally have $\Prob[\calE^C] \leq \delta$.
\end{proof}

\begin{proof}[Proof of Lem.~\ref{lem:hp.suboptimal.arms}]
On event $\calE$, $\Theta_t$ always contains the true model $\btheta$, thus only models with larger optimal value could be selected as the optimistic model $\theta_t = \arg\max_{\theta\in\Theta_t} \mu_*(\theta)$, thus restricting the focus of the algorithm only to the models in $\Theta_+$ and their respective optimal arms.
\end{proof}

\begin{proof}[Proof of Lem.~\ref{l:gaps}]
By definition of $\Theta_+$ we have $\mu_{i_*(\theta)}(\theta) = \mu_*(\theta) > \mu_*(\btheta)$ and by definition of optimal arm we have $\mu_*(\btheta) > \mu_{i_*(\theta)}(\btheta)$, hence $\mu_*(\theta) > \mu_{i_*(\theta)}(\btheta)$. Recalling the definition of model gap, we have $\Gamma_{i_*(\theta)}(\theta) = |\mu_{i_*(\theta)}(\theta) - \mu_{i_*(\theta)}(\btheta)| = \mu_*(\theta) - \mu_{i_*(\theta)}(\btheta)$, where we used the definition of $\mu_*(\theta)$ and the previous inequality. Using the definition of arm gap $\Delta_i$, we obtain
\begin{align*}
\Gamma_{i_*(\theta)}(\theta,\btheta) = \mu_*(\theta) - \mu_{i_*(\theta)}(\btheta) \geq \mu_*(\btheta)- \mu_{i_*(\theta)}(\btheta) = \Delta_{i_*(\theta)}(\btheta),
\end{align*}
which proves the statement.
\end{proof}

\begin{proof}[Proof of Thm.~\ref{thm:mucb.regret}]
We decompose the expected regret as
\begin{align*}
\E[\R_n] = \sum_{i\in \A} \Delta_i \E[T_{i,n}] = \sum_{i\in \A_*(\Theta)} \Delta_i \E[T_{i,n}] \leq n \Prob\{\calE^C\} + \sum_{i\in \A_+} \Delta_i \E[T_{i,n}|\calE],
\end{align*}
where the refinement on the sum over arms follows from Lem.~\ref{lem:suboptimal.arms} and~\ref{lem:hp.suboptimal.arms} and the high probability event $\calE$. In the following we drop the dependency on $\btheta$ and we write $\mu_i(\btheta) = \mu_i$.

We now bound the regret when the correct model is always included in $\Theta_t$. On event $\calE$, only the restricted set of \textit{optimistic} models $\Theta_+ = \{\theta\in\Theta: \mu_*(\theta)\geq \mu_*\}$ is actually used by the algorithm. Thus we need to compute the number of pulls to the suboptimal arms before all the models in $\Theta_+$ are discarded from $\Theta_t$. We first compute the number of pulls to an arm $i$ needed to discard a model $\theta$ on event $\calE$. We notice that
\begin{align*}
\theta\in\Theta_t \Leftrightarrow \{\forall i\in\A, |\mu_i(\theta)-\hmu_{i,t}|\leq \eps_{i,t}\},
\end{align*}
which means that a model $\theta$ is included only when all its means are \textit{compatible} with the current estimates. Since we consider event $\calE$, $|\mu_i-\hmu_{i,t}|\leq \eps_{i,t}$, thus $\theta\in\Theta_t$ only if for all $i\in\A$
\begin{align*}
2 \eps_{i,t} \geq \Gamma_i(\theta, \btheta),
\end{align*}
which corresponds to
\begin{align}\label{e:condition}
T_{i,t-1} \leq \frac{2}{\Gamma_i(\theta,\btheta)^2} \log \bigg(\frac{|\Theta| n^2}{\delta}\bigg),
\end{align}
which implies that if there exists at least one arm $i$ for which at time $t$ the number of pulls $T_{i,t}$ exceeds the previous quantity, then $\forall s>t$ we have $\theta \notin \Theta_t$ (with probability $\Prob(\calE)$). To obtain the final bound on the regret, we recall that the algorithm first selects an optimistic model $\theta_t$ and then it pulls the corresponding optimal arm until the optimistic model is not discarded. Thus we need to compute the number of times the optimal arm of the optimistic model is pulled before the model is discarded. More formally, since we know that on event $\calE$ we have that $T_{i,n}=0$ for all $i\notin\A_+$, the constraints of type~(\ref{e:condition}) could only be applied to the arms $i\in\A_+$. Let $t$ be the last time arm $i$ is pulled, which coincides, by definition of the algorithm, with the last time any of the models in $\Theta_{+,i} = \{\theta\in\Theta_+: i_*(\theta)=i\}$ (i.e., the optimistic models recommending $i$ as the optimal arm) is included in $\Theta_t$. Then we have that $T_{i,t-1}=T_{i,n}-1$ and the fact that $i$ is pulled corresponds to the fact the a model $\theta_i\in\Theta_{+,i}$ is such that
\begin{align*}
\theta_i \in \Theta_t \wedge \forall \theta'\in\Theta_t, \mu_*(\theta_i)>\mu_*(\theta'),
\end{align*}
which implies that (see Eq.~\ref{e:condition})
\begin{align}\label{e:condition2}
T_{i,n} \leq \frac{2}{\min_{\theta\in\Theta_{+,i}}\Gamma_i(\theta,\btheta)^2} \log \bigg(\frac{|\Theta| n^2}{\delta}\bigg) + 1.
\end{align}
where the minimum over $\Theta_{+,i}$ guarantees that all the optimistic models with optimal arm $i$ are actually discarded.\\
Grouping all the conditions, we obtain the expected regret
\begin{align*}
\E[\R_n] \leq K + \sum_{i\in \A_+}  \frac{2\Delta_i(\btheta)}{\min_{\theta\in\Theta_{+,i}}\Gamma_i(\theta,\btheta)^2} \log \big(|\Theta| n^3\big)
\end{align*}
with $\delta=1/n$. Finally we can apply Lem.~\ref{l:gaps} which guarantees that for any $\theta\in\Theta_{+,i}$ the gaps $\Gamma_i(\theta,\btheta)\geq\Delta_i(\btheta)$ and obtain the final statement.
\end{proof}

\textbf{Remark (proof).} The proof of the theorem considers a worst case. In fact, while pulling the optimal arm of the optimistic model $i_*(\theta_t)$ we do not consider that the algorithm might actually discard other models, thus reducing $\Theta_t$ before the optimistic model is actually discarded. More formally, we assume that for any $\theta\in\Theta_t$ not in $\Theta_{+,i}$ the number of steps needed to be discarded by pulling $i_*(\theta_t)$ is larger than the number of pulls needed to discard $\theta_t$ itself, which corresponds to
%
\begin{align*}
\min_{\theta\in\Theta_{+,i}} \Gamma_i^2(\theta,\btheta) \geq \mathop{\max_{\theta\in\Theta^+}}_{\theta\notin\Theta_{+,i}}\Gamma_i^2(\theta,\btheta).
\end{align*}
Whenever this condition is not satisfied, the analysis is suboptimal since it does not fully exploit the structure of the problem and \textit{mUCB} is expected to perform better than predicted by the bound.

\textbf{Remark (comparison to \textit{UCB} with hypothesis testing).}
An alternative strategy is to pair \textit{UCB} with hypothesis testing of fixed confidence $\delta$. Let $\Gamma_{\min}(\btheta) = \min_i\min_\theta \Gamma_i(\theta,\btheta)$, if at time $t$ there exists an arm $i$ such that $T_{i,t} > 2\log(2/\delta)\Gamma_{\min}^2$,
then all the models $\theta\neq\bar\theta$ can be discarded with probability $1-\delta$. Since from the point of view of the hypothesis testing the exploration strategy is unknown, we can only assume that after $\tau$ steps we have $T_{i,\tau}\geq \tau/K$ for at least one arm $i$. Thus after $\tau > 2K \log(2/\delta)/\Gamma_{\min}^2$ steps, the hypothesis testing returns a model $\htheta$ which coincides with $\bar\theta$ with probability $1-\delta$. If $\tau \leq n$, from time $\tau$ on, the algorithm always pulls $I_t =i_*(\htheta)$ and incurs a zero regret with high probability. 
If we assume $\tau \leq n$, the expected regret is
\begin{align*}
\E[\R_n(\text{UCB+Hyp})] \leq O\bigg(\sum_{i\in\A}\frac{\log n\tau}{\Delta_i}\bigg) \leq O\bigg(K\frac{\log n\tau}{\Delta}\bigg).
\end{align*}
We notice that this algorithm only has a mild improvement w.r.t. standard \textit{UCB}. In fact, in \textit{UCB} the big-$O$ notation hides the constants corresponding to the exponent of $n$ in the logarithmic term. This suggests that whenever $\tau$ is much smaller than $n$, then there might be a significant improvement. On the other hand, since $\tau$ has an inverse dependency w.r.t. $\Gamma_{\min}$, it is very easy to build model sets $\Theta$ where $\Gamma_{\min} = 0$ and obtain an algorithm with exactly the same performance as \textit{UCB}.
%

%
\section{Sample Complexity Analysis of \textit{RTP}}\label{app:proofs.rtp}
In this section we provide the full sample complexity analysis of  the \textit{RTP} algorithm. In our analysis we rely on some results of \citet{AnandkumarRHKT12}.  \citet{AnandkumarRHKT12} have  provided perturbation bounds on the error of the orthonormal eigenvectors $\widehat v(\theta)$ and the corresponding  eigenvalues  $\widehat \lambda(\theta)$ in terms of the perturbation error of  the transformed tensor $\epsilon=\|T-\widehat T\|$ \citep[see ][Thm 5.1]{AnandkumarRHKT12}. 
However, this result does not provide us with  the sample complexity  bound on  the estimation error of  model means. Here we complete their analysis  by proving a sample complexity bound on the $\ell_2$-norm of the  estimation error of the means  $\|\mu(\theta)-\widehat \mu(\theta)\|$. 

We follow the following steps in our proof: \textbf{(i)} we bound the error $\epsilon$ in terms of the estimation errors $\epsilon_2:=\|\widehat M_2-M_2\|$ and  $\epsilon_3:=\|\widehat M_3-M_3\|$ (Lem.~\ref{lem:eps.bound.pert}). \textbf{(ii)} we prove high probability bounds on the error $\epsilon_2$ and $\epsilon_3$ using some standard concentration inequality results (Lem. \ref{lem.HighProb.mom}).    The bounds on the errors of the estimates $\widehat v(\theta)$ and   $\widehat \lambda(\theta)$ immediately follow from combining the results  of Lem.~\ref{lem:eps.bound.pert}, Lem. \ref{lem.HighProb.mom} and Thm.~\ref{thm.Anad.RTP}. \textbf{(iii)} Based on these bounds we then prove our main result by bounding the estimation error associated with the inverse transformation $\widehat \mu(\theta)=\widehat \lambda(\theta) \widehat B \widehat v(\theta)$ in high probability. 
%
%
%
%
%

We begin   by recalling the perturbation bound of~\citet{AnandkumarRHKT12}:
\begin{theorem}[\citealp{AnandkumarRHKT12}]
\label{thm.Anad.RTP}
Pick $\eta\in(0,1)$. Define $W:=UD^{-1/2}$,  where $ D\in\mathbb{R}^{m\times m}$ is the diagonal matrix of the $m$ largest eigenvalues of $M_2$ and $U\in\mathbb{R}^{K\times m}$ is the matrix with the eigenvectors associated with the $m$ largest eigenvalues of $ M_2$ as its columns. Then $W$ is a linear mapping which satisfies $W^{\mathsf{T}}M_2 W=\mathbf{I}$. Let $\widehat T=T+E\in \mathbb R^{m\times m\times m}$, where the $3^{\mathrm{rd}}$ order moment tensor $T=M_3( W,  W ,  W )$ is symmetric and orthogonally decomposable in the form of $\sum_{\theta\in \Theta}\lambda(\theta) v(\theta)^{\otimes3}$, where each $\lambda(\theta)>0$  and $\{v(\theta)\}_\theta$  is an orthonormal basis. Define $\epsilon:=\|E\|$ and $\lambda_{\max}=\max_{\theta}\lambda(\theta)$.
Then there exist some constants $C_1,C_2>0$, some polynomial  function $f(\cdot)$, and a permutation $\pi$ on $\Theta$ such that the following holds  w.p. $1-\eta$
\begin{equation*}
\begin{aligned}
\|v(\theta)-\widehat v (\pi(\theta))\|&\leq 8\epsilon /\lambda(\theta),
\\
|\lambda(\theta)-\widehat \lambda (\pi(\theta))|&\leq 5\epsilon,
\end{aligned}
\end{equation*}
for $\epsilon\leq C_1\frac{\lambda_{\min}}{m}$,  $L>\log(1/\eta)f(k)$ and $N\geq C_2(\log(k)+\log\log(\lambda_{\max}/\epsilon))$, where $N$ and $L$ are the internal parameters of RTP algorithm.

\end{theorem}

For ease of exposition we consider the \textit{RTP} algorithm in asymptotic case, i.e.,  $N,L\to\infty$ and  $\eta\approx 1$.
We now prove bounds on the perturbation error $\epsilon$ in terms of the estimation error $\epsilon_2$ and $\epsilon_3$. This requires bounding the error between $W=UD^{-1/2}$ and $\widehat W=\widehat U \widehat D^{-1/2}$ using the following perturbation bounds on  $\|U-\widehat U\|$,  $\|\widehat D^{-1/2}-D^{-1/2}\|$  and $\|\widehat D^{1/2}-D^{1/2}\|$.

\begin{lemma}
\label{lem:DBound}
Assume that $\epsilon_2\leq 1/2 \min( \Gamma_{\sigma},\sigma_{\min})$, then we have
\begin{equation*}
\|\widehat D^{-1/2}-D^{-1/2}\|\leq\frac{2\epsilon_2}{(\sigma_{\min})^{3/2}},
\quad \text{and} \quad
\|\widehat D^{1/2}-D^{1/2}\|\leq\frac{\epsilon_2}{\sigma_{\max}},
\quad\text{and}\quad
\|\widehat U-U\|\leq \frac{2\sqrt{m}\epsilon_2  }{\Gamma_{\sigma}}.
\end{equation*}
\end{lemma}

\begin{proof}

Here we just prove bounds on $\|\widehat D^{-1/2}-D^{-1/2}\|$ and $\|\widehat U-U\|$. The bound on $\|\widehat D^{-1/2}-D^{-1/2}\|$ can be proven  using a  similar argument to that used for bounding  $\|\widehat D^{1/2}-D^{1/2}\|$.
Let $\widehat\Sigma_m=\{\widehat \sigma_1,\widehat \sigma_2,\dots,\widehat \sigma_m\}$ be the set of $m$ largest  eigenvalues of the matrix  $\widehat M_2$. 
We have
\begin{equation*}
\begin{aligned}
& \|\widehat D^{-1/2}-D^{-1/2}\|\overset{(1)}{=} \max_{1 \leq i\leq m}\left|\sqrt{\frac 1{\sigma_i}}-\sqrt{\frac 1{\widehat \sigma_i}}\right |=\max_{1 \leq i\leq m}\left(\frac{\left|\frac 1{\sigma_i}-\frac 1{\widehat \sigma_i}\right |}{\sqrt{\frac 1{\sigma_i}}+\sqrt{\frac 1{\widehat \sigma_i}}}\right)
 \\
 \leq&\max_{1 \leq i\leq m}\left(\sqrt{ \sigma_i}\left|\frac 1{\sigma_i}-\frac 1{\widehat \sigma_i}\right |\right)
 \leq
\max_{1 \leq i\leq m}\left|\frac {\sigma_i-\widehat \sigma_i}{\sqrt{\sigma_i}\widehat \sigma_i}\right |\overset{(2)}{ \leq}
\frac{\epsilon_2}{\sqrt{\sigma_{\min}}(\sigma_{\min}-\epsilon_2)} \overset{(3)}{\leq}\frac{2\epsilon_2}{(\sigma_{\min})^{3/2}},
  \end{aligned}
\end{equation*}
where in (1) we use the fact that the spectral norm of matrix is its largest singular value, which in case of a diagonal matrix coincides with its biggest element, in (2) we rely on the  result of Weyl \citep[see][Thm. 4.11, p. 204]{stewart1990matrix} for bounding the difference between $\sigma_i$ and $\widehat \sigma_i$, and in (3) we make use of  the assumption that $\epsilon_2\leq 1/2 \sigma_{\min}$.

In the case of  $\|U-\widehat U \|$ we rely on the perturbation bound of~\citet{Wedin72}. This result guarantees that  for any positive definite matrix $A$  the difference between the eigenvectors  of $A$ and the perturbed   $\widehat A$ (also positive definite) is small whenever there is a minimum gap between the eigenvalues of   $\widehat A$ and $A$. More precisely, for any positive definite matrix $A$ and  $\widehat A$ such that $||A-\widehat A||\leq \epsilon_A$, let the minimum eigengap be $\Gamma_{A   \leftrightarrow \widehat A}:= \min_{j\neq i}|\sigma_i-\widehat \sigma_j|$, then we have 
\begin{equation}
\label{eq:Wedin}
\|u_i- \widehat u_i \| \leq \frac{\epsilon_A}{\Gamma_{A \leftrightarrow\widehat A}},
\end{equation}
where $(u_i,\sigma_i)$ is an  eigenvalue/vector pair for the matrix $A$.  Based on this result we now bound the error $\|U-\widehat U\|$
\begin{equation*}
\begin{aligned}
\| U - \widehat U \| &\leq    \| U - \widehat U \|_F \leq   \sqrt{\sum_i\| u_i - \widehat u_i \|^2}\overset{(1)}{\leq}\frac{\sqrt{m} \epsilon_2  }{\Gamma_{M_2 \leftrightarrow\widehat M_2}}\overset{(2)}{\leq} \frac{\sqrt{m}\epsilon_2  }{\Gamma_{\sigma}-\epsilon_2}\overset{(3)}{\leq} \frac{2\sqrt{m}\epsilon_2  }{\Gamma_{\sigma}},
\end{aligned}
\end{equation*}
where in (1) we rely on Eq.~\ref{eq:Wedin} and in (2) we rely on the definition of the gap as well as Weyl's inequality. Finally, in (3) We rely on the fact that $\epsilon_2\leq 1/2  \Gamma_{\sigma}$ for bounding denominator from below. 

Our result also   holds for those cases where the multiplicity of some of the  eigenvalues are greater than $1$. Note that for any eigenvalue $\lambda$ with multiplicity $l$    the  linear combination of  the corresponding eigenvectors $\{v_1,v_2,\dots,v_l\}$ is also an eigenvector of the matrix.  Therefore, in this case it suffices to bound the  difference between the eigenspaces of two matrix. The result of  ~\citet{Wedin72} again applies to this case and bounds the difference between the eigenspaces in terms of the perturbation $\epsilon_2$ and $\Gamma_{\sigma}$.  
\end{proof}

We now  bound  $\epsilon$ in terms of $\epsilon_2$ and $\epsilon_3$.

\begin{lemma}
\label{lem:eps.bound.pert}
Let $\mu_{\max}:=\max_{\theta}\|\mu(\theta)\|$, if $\epsilon_2\leq 1/2 \min( \Gamma_{\sigma},\sigma_{\min})$, then the estimation error  $\epsilon$ is bounded as 
 \begin{equation*}
 \epsilon \leq   \left(\frac{m}{\sigma_{\min}}\right)^{3/2}\left(10\epsilon_2 \left(\frac{1}{\Gamma_{\sigma}}+\frac{1}{\sigma_{\min}}\right)\left( \epsilon_3 + \mu^3_{\max} \right)+\epsilon_3\right).
 \end{equation*}
\end{lemma}

\begin{proof}
%
Based on the definitions of $T$ and $\widehat T$ we have
\begin{equation}
\label{eps.bound1}
\begin{aligned}
\epsilon &=\|T-\widehat T \|=\| M_3(W,W,W)-\widehat M_3(\widehat W, \widehat W, \widehat W) \|
\\
&\leq\|M_3(W,W,W)-\widehat M_3(W,W,W)\|+\|\widehat M_3(W,W,W)-\widehat M_3(W,W,\widehat W)\|
\\
&\quad+\|\widehat M_3(W,W,\widehat W)-\widehat M_3(W,\widehat W,\widehat W)\|
+\|\widehat M_3(W,\widehat W,\widehat W)-\widehat M_3(\widehat W,\widehat W,\widehat W)\|
\\
&=\|E_{M_3}(W,W,W)\|+\|\widehat M_3(W,W,W-\widehat W)\|
+\|\widehat M_3(W,W-\widehat W,\widehat W)\|
\\
&\quad+\|\widehat M_3(W-\widehat W,\widehat W,\widehat W)\|,
\end{aligned}
\end{equation}
where $E_{M_3}=M_3-\widehat M_3$. We now bound the  terms in the r.h.s. of Eq.~\ref{eps.bound1}  in terms of   $\epsilon_3$ and $\epsilon_2$.  We begin by bounding  $\|E_{M_3}(W,W,W)\|$:
\begin{equation}
\label{eq:DWBound}
\begin{aligned}
\|E_{M_3}(W,W,W)\|&\leq\|E_{M_3}\|\|W\|^3\leq\|E_{M_3}\|\|U\|^3\|D^{-1}\|^{3/2}\leq\|E_{M_3}\|\|U\|_F^3\|D^{-1}\|^{3/2}
\\&
\overset{(1)}{=} \left(\frac m{\sigma_{\min}}\right)^{3/2}\|E_{M_3}\|\leq\left(\frac m{\sigma_{\min}}\right)^{3/2}\epsilon_3,
\end{aligned}
\end{equation}
where in (1) we use the fact that $U$ is an orthonormal matrix and $D$ is diagonal.
In the case of  $\|\widehat M_3(W,W,W-\widehat W)\|$ we have
\begin{equation*}
\begin{aligned}
&\|\widehat M_3(W,W,W-\widehat W)\|\leq\|W\|^2\|W-\widehat W\|\| \widehat M_3 \|\leq\|W\|^2\|W-\widehat W\|(\| \widehat M_3 -M_3\|+\|M_3\|)
\\
&\overset{(1)}{\leq}\|W\|^2\|W-\widehat W\|( \epsilon_3 + \mu^3_{\max} )\leq\|W\|^2\|UD^{-1/2}-\widehat U \widehat D^{-1/2}\|( \epsilon_3 + \mu^3_{\max} )
\\
&\leq\|W\|^2(\|(U-\widehat U )D^{-1/2}\|+\|\widehat U (\widehat D^{-1/2}-D^{-1/2})\|)\left( \epsilon_3 + \mu^3_{\max} \right)
\\
&\leq\|W\|^2\left(\frac{\|U-\widehat U \|}{\sqrt{\sigma_{\min}}}+\sqrt{m}\|\widehat D^{-1/2}-D^{-1/2}\|\right)\left( \epsilon_3 + \mu^3_{\max} \right).
\end{aligned}
\end{equation*}
where in (1) we use the definition of $M_3$ as a linear combination of the tensor product of the means $\mu(\theta)$.
This result combined with the result of Lem.~\ref{lem:DBound} and  the fact that  $\|W\|\leq  \sqrt{m/\sigma_{\min}}$ (see Eq.~\ref{eq:DWBound}) implies that
\begin{equation}
\label{eq:MWDiff1}
\begin{aligned}
\|\widehat M_3(W,W,W-\widehat W)\|  \leq \frac{m}{\sigma_{\min}}\left(\frac{2\sqrt{m}\epsilon_2}{\Gamma_{\sigma}\sqrt{\sigma_{\min}}}+\frac{2\sqrt{m}\epsilon_2}{(\sigma_{\min})^{3/2}}\right)\left( \epsilon_3 + \mu^3_{\max} \right)
&\\ \leq 2\epsilon_2\left(\frac{m}{\sigma_{\min}}\right)^{3/2}\left(\frac{1}{\Gamma_{\sigma}}+\frac{1}{\sigma_{\min}}\right)\left( \epsilon_3 + \mu^3_{\max} \right).
\end{aligned}
\end{equation}
Likewise one can prove the following perturbation bounds for $\widehat M_3(W,W-\widehat W,\widehat W)$ and $\widehat M_3(W,W-\widehat W,\widehat W)$:
\begin{equation}
\label{eq:MWDiff2}
\begin{aligned}
\|\widehat M_3(W,W-\widehat W,\widehat W)\|&\leq  2\sqrt{2}\epsilon_2\left(\frac{m}{\sigma_{\min}}\right)^{3/2}\left(\frac{1}{\Gamma_{\sigma}}+\frac{1}{\sigma_{\min}}\right)\left( \epsilon_3 + \mu^3_{\max} \right)
\\
\|\widehat M_3(W-\widehat W,\widehat W,\widehat W)\|& \leq4\epsilon_2\left(\frac{m}{\sigma_{\min}}\right)^{3/2}\left(\frac{1}{\Gamma_{\sigma}}+\frac{1}{\sigma_{\min}}\right)\left( \epsilon_3 + \mu^3_{\max} \right).
\end{aligned}
\end{equation}
The result then follows by plugging  the bounds of Eq.~\ref{eq:DWBound}, Eq.~\ref{eq:MWDiff1} and Eq.~\ref{eq:MWDiff2}  into  Eq.~\ref{eps.bound1}.
\end{proof}

We now prove high-probability bounds on $\epsilon_3$ and $\epsilon_2$ when  $M_2$ and $M_3$ are estimated by sampling.

\begin{lemma}
\label{lem.HighProb.mom}
For any $\delta\in(0,1)$, if $\widehat M_2$ and $\widehat M_3$ are computed with samples from $j$ episodes, then we that with probability $1-\delta$:
\begin{equation*}
\epsilon_3\leq K^{1.5}\sqrt{\frac{6\log(2K/\delta)}{j}}\qquad\text{and}\qquad\epsilon_2 \leq 2K\sqrt{\frac{\log(2K/\delta)}{j}}.
\end{equation*}
 \end{lemma}

\begin{proof}
Using some norm inequalities for the tensors we obtain
\begin{equation*}
\epsilon_3 =\|M_3-\widehat M_3\| \leq K^{1.5} \|M_3-\widehat M_3\|_{\max}= K^{1.5} \max_{i,j,x}|[M_3]_{i,j,x}-[\widehat M_3]_{i,j,x}|.
\end{equation*}
A similar  argument  leads to the bound of  $K \max_{i,j}|[M_2]_{i,j}-[\widehat M_2]_{i,j}|$ on $\epsilon_2$. One can easily show that,  for every $ 1\leq i,j,x\leq K$, the term   $[M_3]_{i,j,x}-[\widehat M_3]_{i,j,x}$ and $[M_3]_{i,j,x}-[\widehat M_3]_{i,j,x}$   can be expressed as a sum  of martingale differences with the maximum value $1/j$.   The result then follows by applying the Azuma's inequality~\citep[e.g., see][appendix, pg. 361]{cesa2006prediction}  and taking the  union bound. 
\end{proof}

We now draw our attention to the proof of our main result.

\begin{proof}[\textbf{Proof of Thm.~\ref{thm:mom}}]
We begin by deriving the condition of Eq. \ref{eq:bound.min.samp}. The assumption on $\epsilon_2$ in Lem. \ref{lem:eps.bound.pert}  and the result of Lem. \ref{lem.HighProb.mom} hold at the same time, w.p. $1-\delta$, if the following inequality holds  
\begin{equation*}
 2K\sqrt{\frac{\log(2K/\delta)}{j}} \leq 1/2\min(\Gamma_{\sigma},\sigma_{\min}).            
\end{equation*}
By solving the bound w.r.t. $j$ we obtain 
\begin{equation}
\label{eq:bound.J.eps2}
j\geq \frac{16K^2 \log(2K/\delta)}{ \min(\Gamma_{\sigma},\sigma_{\min})^2}.
\end{equation}
A similar argument applies in the case of the assumption on $\epsilon$ in Thm.~\ref{thm.Anad.RTP}. The results of Thm.~\ref{thm.Anad.RTP} and Lem. \ref{lem:eps.bound.pert}   hold at the same time if we have 
\begin{equation*}
\eps \leq  \left(\frac{m K}{\sigma_{\min}}\right)^{3/2}\left(20\epsilon_2 \left(\frac{1}{\Gamma_{\sigma}}+\frac{1}{\sigma_{\min}}\right)+\epsilon_3\right)\leq C_1\frac{\lambda_{\min}}{m},
\end{equation*}
where in the first inequality we used that $\eps_3 \leq K^{3/2}$ and $\mu_{\max}^3 \leq K^{3/2}$ by their respective definitions. 
This combined with high probability bounds of Lem. \ref{lem.HighProb.mom} on $\epsilon_1$ and $\epsilon_2$ implies
\begin{equation*}
  \left(\frac{m }{\sigma_{\min}}\right)^{1.5}\left(20K^{2.5}\sqrt{\frac{\log(4K/\delta)}{j}}\left(\frac{1}{\Gamma_{\sigma}}+\frac{1}{\sigma_{\min}}\right)+K^{1.5}\sqrt{\frac{6\log(4K/\delta)}{j}}\right)\leq C_1\frac{\lambda_{\min}}{m}.
\end{equation*}
By solving this bound w.r.t. $j$ (and some simplifications) we obtain w.p. $1-\delta$
\begin{equation*} 
j\geq \frac{43^2 m^5 K^6\log(4K/\delta)}{C_1\sigma^3_{\min}\lambda^2_{\min}}\left(\frac{1}{\Gamma_{\sigma}}+\frac{1}{\sigma_{\min}}\right)^2.
\end{equation*}
Combining this result with that of Eq.\ref{eq:bound.J.eps2} and taking the union bound leads to the bound of  Eq.~\ref{eq:bound.min.samp} on the minimum number of samples.

We now draw our attention to the main result of the theorem. We begin by bounding $\|\mu(\theta)-\widehat \mu(\pi(\theta))\|$ in terms of estimation error term $\epsilon_3$ and $\epsilon_2$: 
\begin{equation}
\label{eq:mubound.fin}
\begin{aligned}
&\|\mu(\theta)-\widehat \mu(\pi(\theta))\| = \|\lambda(\theta) B  v(\theta)- \widehat \lambda(\pi(\theta)) \widehat B  \widehat v(\pi(\theta))\|
\\ 
\leq&
\|(\lambda(\pi(\theta))-\widehat\lambda(\theta) ) B v(\pi(\theta))\|+\|\widehat\lambda(\theta) (B-\widehat B) v(\pi(\theta))\| +\| \widehat \lambda(\theta)  \widehat B (v(\pi(\theta))-\widehat v(\theta) )\| 
\\
\leq&
|\lambda(\theta)-\widehat\lambda(\pi(\theta)) | \|B \|+\widehat\lambda(\pi(\theta))\| B-\widehat B\| +\widehat \lambda(\pi(\theta)) \| \widehat B\| \|v(\theta)-\widehat v(\pi(\theta)) \| ,
\end{aligned}
\end{equation}
where in the last line we rely on the fact that both $v(\theta)$ and $\widehat v(\pi(\theta))$ are normalized vectors. 
We first bound the term $\|B-\widehat B\|$:
\begin{equation*}
\begin{aligned}
\|B-\widehat B\|&= \|UD^{1/2}-\widehat U \widehat D^{1/2}\|\leq  \|(U-\widehat U)D^{1/2}\|  +\|\widehat U( D^{1/2}-\widehat D^{1/2})\|
\\
&\overset{(1)}{\leq}\frac{2\sqrt{m}\epsilon_2\sigma_{\max}}{\Gamma_{\sigma} }  +\frac{\sqrt{m}\epsilon_2}{\sigma_{\max}}\leq \sqrt{m}\epsilon_2\left(\frac{2\sigma_{\max}}{\Gamma_{\sigma}}+\frac{1}{\sigma_{\max}}\right),
\end{aligned}
\end{equation*}
where in (1) we make use of the result of  Lem.~\ref{lem:DBound}. Furthermore, we have
\begin{equation*}
\begin{aligned}
\|\widehat B\|&= \|\widehat U \widehat D^{1/2}\|\leq \sqrt{m \widehat \sigma_{\max}}\leq \sqrt{m} (\sigma_{\max}^{1/2}+\epsilon_2^{1/2})\leq \sqrt{m}(\sigma_{\max}^{1/2}+\sigma_{\min}^{1/2}) \leq  \sqrt{2m\sigma_{\max}},
\end{aligned}
\end{equation*}
where we used the condition on $\epsilon_2$.
This combined with Eq.\ref{eq:mubound.fin} and the result of Thm~\ref{thm.Anad.RTP} and Lem. \ref{lem:eps.bound.pert} implies

\begin{equation*}
\begin{aligned}
\|\mu&(\pi(\theta))-\widehat \mu(\theta)\| \\
&\overset{(1)}{\leq}  5\sqrt{m\sigma_{\max}} \epsilon  +\sqrt{m}\epsilon_2\left(\lambda(\theta)+\epsilon\right)\left(\frac{2\sigma_{\max}}{\Gamma_{\sigma}}+\frac{1}{\sigma_{\max}}\right)+\frac{8\epsilon}{\lambda(\theta)}\sqrt{2m\sigma_{\max}}\left(\lambda(\theta)+\epsilon\right) \\
&\overset{(2)}{\leq}  5\sqrt{m\sigma_{\max}} \epsilon  +\sqrt{m}\epsilon_2\left(\lambda(\theta)+5C_1\frac{\sigma_{\min}}{m}\right)\left(\frac{2\sigma_{\max}}{\Gamma_{\sigma}}+\frac{1}{\sigma_{\max}}\right) +8\sqrt{2m\sigma_{\max}}\left(1+5C_1\frac{\sigma_{\min}}{m}\right)\epsilon
\\
&\leq 5\sqrt{m\sigma_{\max}} \left(\frac{m}{\sigma_{\min}}\right)^{3/2}\left(10\epsilon_2 \left(\frac{1}{\Gamma_{\sigma}}+\frac{1}{\sigma_{\min}}\right)\left( \epsilon_3 + \mu^3_{\max} \right)+\epsilon_3\right)  
\\
&
\quad+\sqrt{m}\epsilon_2\left(\lambda(\theta)+5C_1\frac{\sigma_{\min}}{m}\right)\left(\frac{2\sigma_{\max}}{\Gamma_{\sigma}}+\frac{1}{\sigma_{\max}}\right)
\\
&
\quad+8\sqrt{2m\sigma_{\max}}\left(1+\frac{5C_1}{m}\right)\left(\frac{m}{\sigma_{\min}}\right)^{3/2}\left(10\epsilon_2 \left(\frac{1}{\Gamma_{\sigma}}+\frac{1}{\sigma_{\min}}\right)\left( \epsilon_3 + \mu^3_{\max} \right)+\epsilon_3\right).
\end{aligned}
\end{equation*}
where in (1) we used $||B|| \leq \sqrt{m\sigma_{\max}}$, the bound on $\widehat \lambda(\pi(\theta)) \leq \lambda(\theta) + 5\epsilon$, $\|v(\theta)-\widehat v(\pi(\theta)) \| \leq 8\epsilon / \lambda(\theta)$, in (2) we used $\lambda(\theta) = 1/\sqrt{\rho(\theta)} \geq 1$ and the condition that $\eps \leq 5C_1\sigma_{\min}/m$. 
The result then follows by combining this bound with the high probability bound of Lem. \ref{lem.HighProb.mom} and taking union bound as well as collecting the terms.
\end{proof}


\section{Proofs of Section~\ref{ss:umucb.regret}}\label{app:proofs.umucb}

\begin{lemma}\label{l:umucb.suboptimal.arms}
At episode $j$, the arms $i\notin \A_*^j(\Theta; \btheta^j)$ are never pulled, i.e., $T_{i,n}=0$.
\end{lemma}

\begin{lemma}\label{l:umucb.high.prob}
If \textit{umUCB} is run with
\begin{align}\label{eq:umucb.eps.def}
\eps_{i,t} = \sqrt{\frac{1}{2T_{i,t-1}} \log \bigg(\frac{2mKn^2}{\delta}\bigg)}, \quad\quad \eps^j = C(\Theta)\sqrt{\frac{1}{j} \log \bigg(\frac{2mKJ}{\delta}\bigg)},
\end{align}
where $C(\Theta)$ is defined in Thm.~\ref{thm:mom}, then the event $\calE = \calE_1 \cap \calE_2$ is such that $\Prob[\calE] \geq 1-\delta$ where $\calE_1 = \{\forall \theta, t, i, |\hmu_{i,t}-\mu_i(\theta)| \leq \eps_{i,t}\}$ and $\calE_2 =\{\forall j,\theta,i, |\hmu_i^j(\theta)-\mu_i(\theta)|\leq \eps^j\}$.
\end{lemma}

Notice that the event $\calE$ implies that for any episode $j$ and step $t$, the actual model is always in the active set, i.e., $\btheta^j\in\Theta_t^j$.

\begin{lemma}\label{l:umucb.suboptimal.arms.hp}
At episode $j$, all the arms $i\notin \A_+^j(\Theta_+^j(\btheta^j); \btheta^j)$ are never pulled on event $\calE$, i.e., $T_{i,n}=0$ with probability $1-\delta$.
\end{lemma}

\begin{lemma}\label{l:umucb.pulls.ucb}
At episode $j$, the arms $i \in \A_+^j(\Theta_+^j(\btheta^j); \btheta^j)$ are never pulled more than with a \textit{UCB} strategy, i.e., 
\begin{align}\label{eq:umucb.pulls.ucb}
T_{i,n}^j \leq \frac{2}{\Delta_i(\btheta^j)^2} \log \bigg(\frac{2mKn^2}{\delta}\bigg)+1,
\end{align}
with probability $1-\delta$.
\end{lemma}

Notice that for \textit{UCB} the logarithmic term in the previous statement would be $\log (Kn^2/\delta)$ which would represent a negligible constant fraction improvement w.r.t. \textit{umUCB} whenever the number of models is of the same order of the number of arms.

\begin{lemma}\label{l:pulls.discardable.models}
At episode $j$, for any model $\theta\in(\Theta^j_+(\btheta^j)-\tTheta^j(\btheta^j))$ (i.e., an optimistic model that can be discarded), the number of pulls to any arm $i\in\A_+^j(\theta; \btheta^j)$ needed before discarding $\theta$ is
\begin{align}\label{eq:pulls.discardable.models}
T_{i,n}^j \leq \frac{1}{2\big(\Gamma_i(\theta, \btheta^j)/2 - \eps^j\big)^2} \log \bigg(\frac{2mKn^2}{\delta}\bigg)+1,
\end{align}
with probability $1-\delta$.
\end{lemma}

\begin{proof}[Proof of Lem.~\ref{l:umucb.suboptimal.arms}]
We first notice that the algorithm only pulls arms recommended by a model $\theta\in\Theta^j_t$. Let $\hi_*(\theta) = \arg\max_i B_t^j(i; \theta)$ with $\theta\in\Theta_t^j$, and $i\in\A_*^j(\theta; \btheta^j)$. According to the selection process, we have
\begin{align*}
B_t^j(i; \theta) < B_t^j(\hi_*; \theta).
\end{align*}
Since $\theta\in\Theta_t^j$ we have that for any $i$, $|\hmu_{i,t}-\hmu_i^j(\theta)| \leq \eps_{i,t} + \eps^j$ which leads to $\hmu_i^j(\theta) -\eps^j \leq \hmu_{i,t}+\eps_{i,t}$. 
Since $\hmu_i^j(\theta) -\eps^j \leq \hmu_i^j(\theta) +\eps^j$, then we have that
\begin{align*}
\hmu_i^j(\theta) -\eps^j \leq \min\{\hmu_{i,t}+\eps_{i,t},\hmu_i^j(\theta) +\eps^j\} = B_t^j(i; \theta).
\end{align*}
Furthermore from the definition of the $B$-values we deduce that
\begin{align*}
B_t^j(\hi_*; \theta) \leq \hmu_{\hi_*}^j(\theta) + \eps^j.
\end{align*}
Bringing together the previous inequalities, we obtain
\begin{align*}
\hmu_i^j(\theta) -\eps^j \leq \hmu_{\hi_*}^j(\theta) + \eps^j.
\end{align*}
which is a contradiction with the definition of non-dominated arms $\A_*^j(\Theta; \btheta^j)$.
\end{proof}

\begin{proof}[Proof of Lem.~\ref{l:umucb.high.prob}]
The probability of $\calE_1$ is computed in Lem.~\ref{l:mucb.high.prob} with the difference that now we need an extra union bound over all the models and that the union bound over the arms cannot be restricted to the number of models. The probability of $\calE_2$ follows from Thm.~\ref{thm:mom}.
\end{proof}

\begin{proof}[Proof of Lem.~\ref{l:umucb.suboptimal.arms.hp}]
We first recall that on event $\calE$, at any episode $j$, the actual model $\btheta^j$ is always in the active set $\Theta_t^j$. If an arm $i$ is pulled, then according to the selection strategy, there exists a model $\theta\in\Theta_t$ such that
\begin{align*}
B_t^j(i; \theta) \geq B_t^j(\hi_*(\btheta^j); \btheta^j).
\end{align*}
Since $\hi_*(\btheta^j) = \arg\max_i B_t^j(i; \btheta^j)$, then $B_t^j(\hi_*(\btheta^j); \btheta^j) \geq B_t^j(i_*(\btheta^j); \btheta^j)$ where $i_*(\btheta^j)$ is the true optimal arm of $\btheta^j$. By definition of $B(i; \theta)$, on event $\calE$ we have that  $B_t^j(i_*(\btheta^j); \btheta^j) \geq \mu_*(\btheta^j)$ and that $B_t^j(i; \theta) \leq \hmu_i^j(\theta)+\eps^j$. Grouping these inequalities we obtain
\begin{align*}
\hmu_i^j(\theta)+\eps^j \geq \mu_*(\btheta^j),
\end{align*}
which, together with Lem.~\ref{l:umucb.suboptimal.arms}, implies that $i\in\A_+^j(\theta; \btheta^j)$ and that this set is not empty, which corresponds to $\theta\in\Theta_+^j(\btheta^j)$.
\end{proof}

\begin{proof}[Proof of Lem.~\ref{l:umucb.pulls.ucb}]
Let $t$ be the last time arm $i$ is pulled ($T_{i,t-1} = T_{i,n}+1$), then according to the selection strategy we have
\begin{align*}
B_t^j(i; \theta_t^j) \geq B_t^j(\hi_*(\btheta^j); \btheta^j) \geq B_t^j(i_*; \btheta^j),
\end{align*}
where $i_* = i_*(\btheta^j)$.
Using the definition of $B$, we have that on event $\calE$
\begin{align*}
B_t^j(i_*(\btheta^j); \btheta^j) = \min \big\{ (\hmu_{i_*}^j(\btheta^j) + \eps^j); (\hmu_{i_*,t} + \eps_{i_*,t}) \big\} \geq \mu_*(\btheta^j)
\end{align*}
and
\begin{align*}
B_t^j(i; \theta_t^j) \leq \hmu_{i,t} + \eps_{i,t} \leq \mu_{i}(\btheta^j) + 2\eps_{i,t}.
\end{align*}
Bringing the two conditions together we have
\begin{align*}
\mu_i(\btheta^j) + 2\eps_{i,t} \geq \mu_*(\btheta^j) \Rightarrow 2\eps_{i,t} \geq \Delta_i(\btheta^j),
\end{align*}
which coincides with the (high-probability) bound on the number of pulls for $i$ using a \textit{UCB} algorithm and leads to the statement by definition of $\eps_{i,t}$.
\end{proof}

\begin{proof}[Proof of Lem.~\ref{l:pulls.discardable.models}]
According to Lem.~\ref{l:umucb.suboptimal.arms.hp}, a model $\theta$ can only propose arms in $\A_+^j(\theta; \btheta^j)$. Similar to the analysis of \textit{mUCB}, $\theta$ is discarded from $\Theta^j_t$ with high probability after $t$ steps and $j$ episodes if
\begin{align*}
2(\eps_{i,t} + \eps^j) \leq \Gamma_i(\theta, \btheta^j).
\end{align*}
At round $j$, if $\eps^j \geq \Gamma_i(\theta,\btheta^j)/2$ then the algorithm will never be able to pull $i$ enough to discard $\theta$ (i.e., the uncertainty on $\theta$ is too large), but since $i\in\A_*^j(\theta; \btheta^j)$, this corresponds to the case when $\theta \in \tTheta^j(\btheta^j)$. Thus, the condition on the number of pulls to $i$ is derived from the inequality
\begin{align*}
\eps_{i,t} \leq \Gamma_i(\theta, \btheta^j)/2 - \eps^j.
\end{align*}
\end{proof}


\section{Related Work}\label{s:related}

As discussed in the introduction, transfer in online learning has been rarely studied. In this section we review possible alternatives and a series of settings which are related to the problem we consider in this paper.

\noindent \textbf{Models estimation.} 
Although in \textit{tUCB} we use \textit{RTP} for the estimation of the model means, a wide number of other algorithms could be used, in particular those based on the method of moments (MoM).
Recently a great deal of progress has been made  regarding the problem of parameter estimation in    LVM  based on  the method of moments approach (MoM)~\citep{AnandkumarHK12,AnandkumarFHKL12,AnandkumarRHKT12}. 
The main idea of \textit{MoM}  is to match the empirical moments of the data with the model parameters that give rise to nearly the  same corresponding population quantities. In general, matching the model parameters to the observed moments may require solving systems of high-order polynomial equations which is often computationally prohibitive. However, for a rich class of  LVMs, it is possible to efficiently estimate the parameters only based on the low-order moments (up to the third order)~\citep{AnandkumarHK12}. 
Prior to \textit{RTP} various scenarios for \textit{MoM} are considered in the literature for different classes of LVMs  using different  linear algebra techniques to deal with the empirical moments~\cite{AnandkumarHK12,AnandkumarFHKL12}. The variant introduced in \cite[Algorithm B]{AnandkumarHK12}  recovers the matrix of the means $\{\mu(\theta)\}$ up to  a permutation in columns without any knowledge of $\rho$.  Also, theoretical guarantees  in the form of  sample complexity bounds with polynomial dependency on the parameters of interest have been provided for this algorithm.  
The excess correlation analysis (ECA) (Alg. 5 in \citet{AnandkumarFHKL12}) generalizes the idea of the \textit{MoM} to the case that $\rho$ is not fixed anymore but sampled from some Dirichlet distribution.  The parameters of this   Dirichlet distribution is not  to be  known by the learner.\footnote{ We only need to know sum of the parameters of the Dirichlet  distribution $\alpha_0$.} In  this case again we can apply a variant of \textit{MoM} to recover the models. 


\noindent \textbf{Online Multi-task.} 
In the online multi-task learning the task change at each step ($n=1$) but at the end of each step both the true label (in the case of online binary classification) and the identity of the task are revealed. A number of works~\citep{dekel2006online,saha2011online,cavallanti2010linear,lugosi2009online} focused on this setting and showed how the samples coming from different tasks can be used to perform multi-task learning and improve the worst-case performance of an online learning algorithm compared to using all the samples separately.

\noindent \textbf{Contextual Bandit.} In contextual bandit~\citep[e.g., see][]{agarwal2012contextual,langford2007epoch}, at each step the learner observes a context $x_t$ and has to choose the arm which is best for the context. The contexts belong to an arbitrary (finite or continuous) space and are drawn from a stationary distribution. This scenario resembles our setting where tasks arrive in a sequence and are drawn from a $\rho$. The main difference is that in our setting the learner does not observe explicitly the context and it repeatedly interact with that context for $n$ steps. Furthermore, in general in contextual bandits some similarity between contexts is used, while here the models are completely independent.

\noindent \textbf{Non-stationary Bandit.} 
When the learning algorithm does not know when the actual change in the task happens, then the problem reduces to learning in a piece-wise stationary environment. \citet{garivier2011on-upper-confidence} introduces a modified version of \textit{UCB} using either a sliding window or discounting to \textit{track} the changing distributions and they show, when optimally tuned w.r.t. the number of switches $R$, it achieves a (worst-case) expected regret of order $O(\sqrt{TR})$ over a total number of steps $T$ and $R$ switches. Notice that this could be also considered as a partial transfer algorithm. Even in the case when  the switch is directly observed, if $T$ is too short to learn from scratch and to identify similarity with other previous tasks, one option is just to transfer the averages computed before the switch. This clearly introduces a transfer bias that could be smaller than the regret cumulated in the attempt of learning from scratch. This is not surprising since transfer is usually employed whenever the number of samples that can be collected from the task at hand is relatively small. If we applied this algorithm to our setting $T=nJ$ and $R=J$, the corresponding performance would be $O(J\sqrt{n})$, which matches the worst-case performance of \textit{UCB} (and \textit{tUCB} as well) on $J$ tasks. This result is not surprising since the advantage of knowing the switching points (every $n$ steps) could always be removed by carefully choosing the worst possible tasks. Nonetheless, whenever we are not facing a worst case, the non-stationary \textit{UCB} would have a much worse performance than \textit{tUCB}.


\section{Numerical Simulations}\label{app:plus.experiment}

\begin{table}[ht]
\begin{center}
\begin{small}
\begin{tabular}{|r|c|c|c|c|c|c|c|}
\hline
 & Arm1 & Arm2 & Arm3 & Arm4 & Arm5 & Arm6 & Arm7 \\
\hline
$\theta_1$  & 0.9  & 0.75  & 0.45  & 0.55  & 0.58  & 0.61  & 0.65 \\
$\theta_2$  & 0.75  & 0.89  & 0.45  & 0.55  & 0.58  & 0.61  & 0.65 \\
$\theta_3$  & 0.2  & 0.23  & 0.45  & 0.35  & 0.3  & 0.18  & 0.25 \\
$\theta_4$  & 0.34  & 0.31  & 0.45  & 0.725  & 0.33  & 0.37  & 0.47 \\
$\theta_5$  & 0.6  & 0.5  & 0.45  & 0.35  & 0.95  & 0.9  & 0.8 \\
\hline
\end{tabular}
\end{small}
\par\vspace{-0.05in}
\caption{Models.}
\label{t:models}
\end{center}
\end{table}

\begin{table}[ht]
\begin{center}
\begin{small}
\begin{tabular}{|r|c|c|c|}
\hline
& \textit{UCB} & \textit{UCB+} & \textit{mUCB} \\
\hline
$\theta_1$ & 22.31 & 14.87 & 2.33\\
$\theta_2$ & 23.32 & 15.58 & 8.48\\
$\theta_3$ & 33.91 & 25.21 & 2.08\\
$\theta_4$ & 17.91 & 11.17 & 3.48\\
$\theta_5$ & 35.41 & 8.76 & 0\\
\hline
\hline
avg & 26.57 & 15.11 & 3.27\\
\hline
\end{tabular}
\end{small}
\par\vspace{-0.05in}
\caption{Complexity of \textit{UCB}, \textit{UCB+}, and \textit{mUCB}.}
\label{t:complexity}
\end{center}
\end{table}

\begin{figure}[t]
\begin{center}
\begin{minipage}[b]{0.48\linewidth}
\vspace{0pt}
\centering
\includegraphics[trim=0.5cm 0cm 1cm 0cm, clip=true,width=0.85\textwidth]{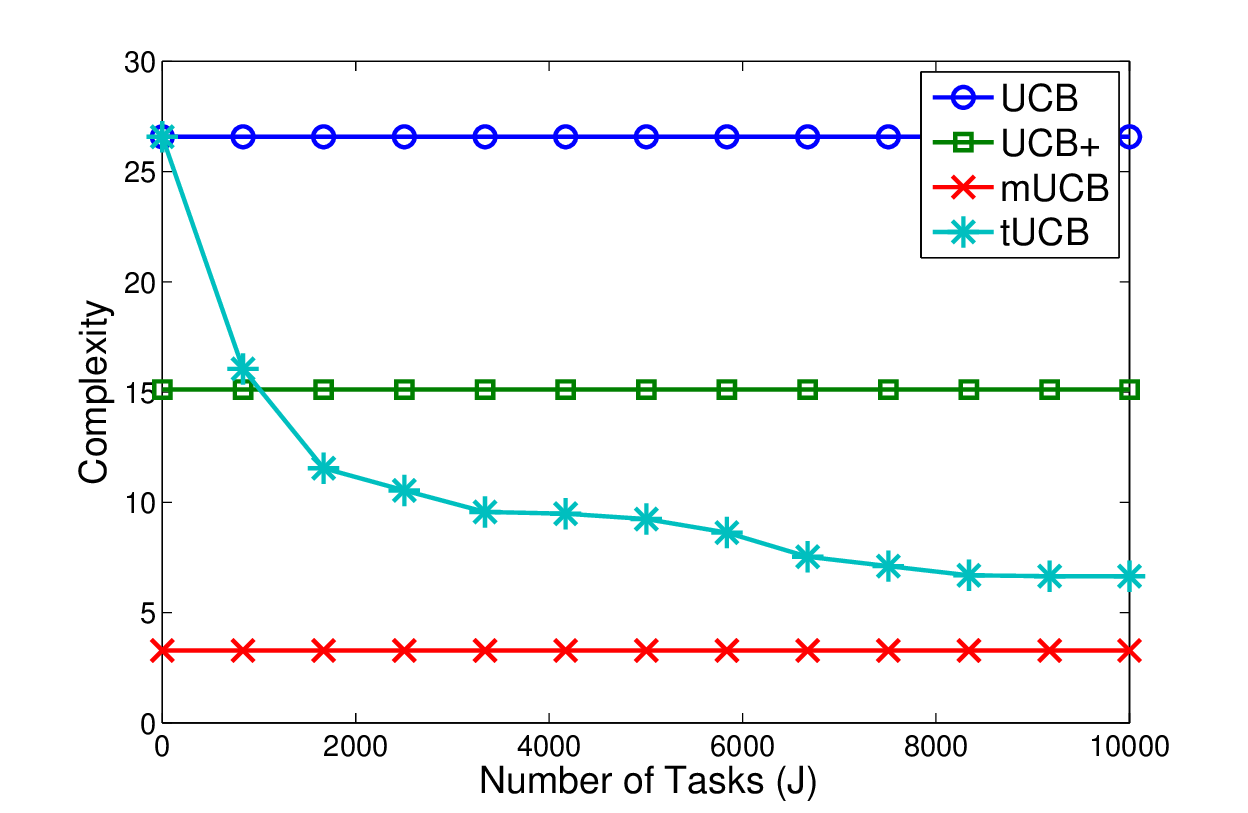}
\par\vspace{-0.0in}
\end{minipage}
\hspace{0.05in}
\begin{minipage}[b]{0.48\linewidth}
\vspace{0pt}
\centering
\includegraphics[trim=0.5cm 0cm 1cm 0cm, clip=true,width=0.85\textwidth]{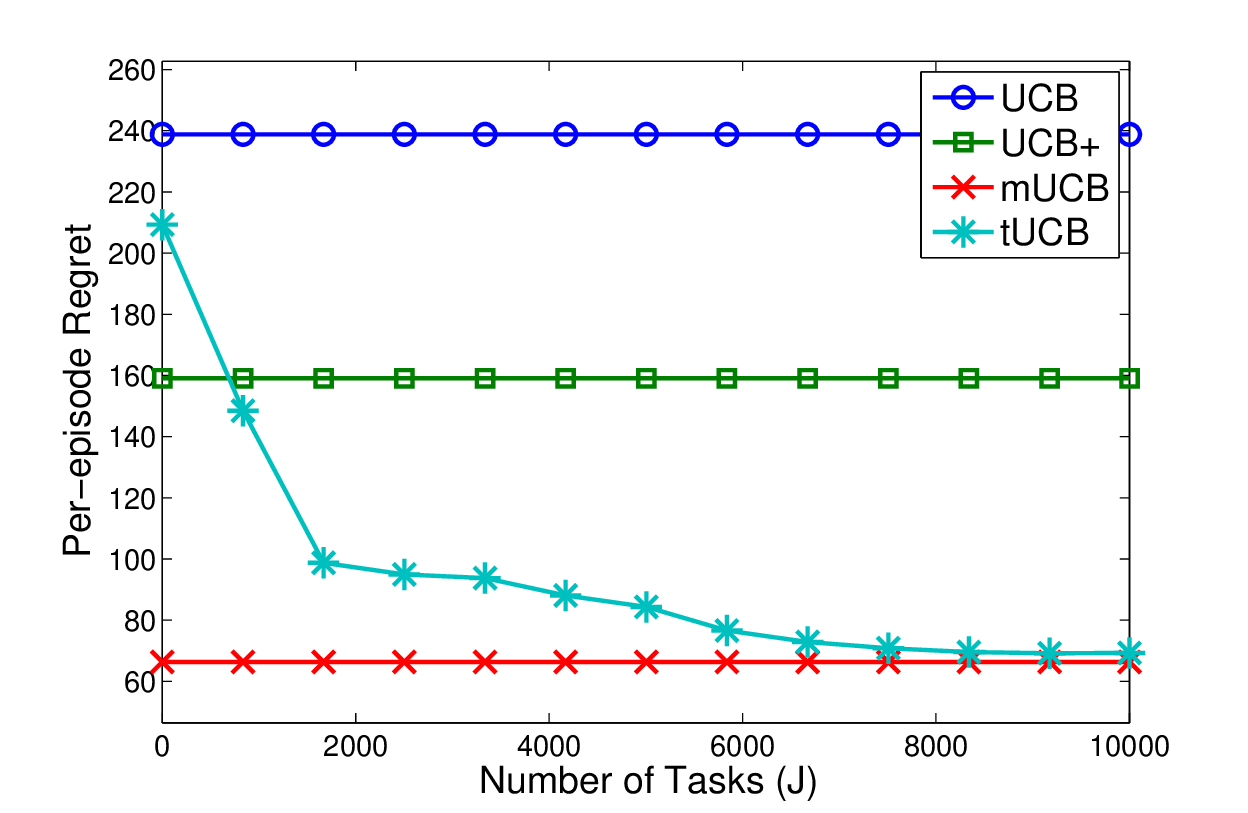}
\par\vspace{-0.0in}
\end{minipage}
\caption{Complexity and per-episode regret of \textit{tUCB} over tasks.}
\label{f:extended}
\end{center}
\end{figure}

In Table~\ref{t:models} we report the actual values of the means of the arms of the models in $\Theta$, while in Table~\ref{t:complexity} we compare the complexity of 
\textit{UCB}, \textit{UCB+}, and \textit{mUCB}, for all the different models and on average. Finally, the graphs in Fig.~\ref{f:extended} are an extension up to $J=10000$ of the performance of \textit{tUCB} for $n=5000$ reported in the main text.

\end{document}